%% file: ms.tex
\documentclass{article}

\usepackage{microtype}
\usepackage{graphicx}
\usepackage{booktabs} 

\usepackage{hyperref}

\usepackage{algorithm}
\usepackage[noend]{algorithmic}

\usepackage[accepted]{icml2024}

\usepackage{amsmath}
\usepackage{amssymb}
\usepackage{mathtools}
\usepackage{amsthm}

\usepackage[capitalize,noabbrev]{cleveref}

\theoremstyle{plain}

\newtheorem{corollary}{Corollary}
\newtheorem{assumption}{Assumption}
\newtheorem{definition}{Definition}
\newtheorem{lemma}{Lemma}
\newtheorem{example}{Example}
\newtheorem{remark}{Remark}

\usepackage[textsize=tiny]{todonotes}
\newcommand{\purechildren}{\text{PCh}_{\graph}}
\newcommand{\node}[1]{{#1}}
\newcommand{\set}[1]{\mathbf{#1}}
\newcommand{\setset}[1]{\mathcal{#1}}
\newcommand{\graph}{\mathcal{G}}
\newcommand{\parents}{\text{Pa}_{\graph}}

\newcommand{\descendant}{\text{De}_{\graph}}
\newcommand{\children}{\text{Ch}_{\graph}}

\DeclareMathOperator*{\argmin}{arg\,min}
\DeclareMathOperator{\tr}{tr}
\DeclareMathOperator{\supp}{supp}
\DeclareMathOperator{\diag}{diag}
\DeclareMathOperator{\score}{score}

\newcommand{\subjectto}{\mathrm{subject\ to}}
\crefname{assumption}{Assumption}{Assumptions}
\newcommand{\EmpCov}{S}
\usepackage{multirow}
\makeatletter

\newcommand{\algorithmicforeach}{\textbf{foreach}}
\newcommand{\FOREACH}[2][default]{%
  \ALC@it\algorithmicforeach\ #2\ \algorithmicdo%
    \ALC@com{#1}\begin{ALC@for}}
\makeatother
\usepackage{subfig}
\usepackage{thmtools,thm-restate}
\usepackage{enumitem}

\makeatletter
\AddToHook{cmd/appendix/before}{\def\cref@section@alias{appendix}\def\cref@subsection@alias{appendix}}
\makeatother

\icmltitlerunning{Score-Based Causal Discovery of Latent Variable Causal Models}

\begin{document}

\twocolumn[
\icmltitle{Score-Based Causal Discovery of Latent Variable Causal Models}



\icmlsetsymbol{equal}{*}

\begin{icmlauthorlist}
\icmlauthor{Ignavier Ng}{equal,cmu}
\icmlauthor{Xinshuai Dong}{equal,cmu}
\icmlauthor{Haoyue Dai}{cmu}
\icmlauthor{Biwei Huang}{ucsd}
\icmlauthor{Peter Spirtes}{cmu}
\icmlauthor{Kun Zhang}{cmu,mbzuai}
\end{icmlauthorlist}

\icmlaffiliation{cmu}{Carnegie Mellon University}
\icmlaffiliation{ucsd}{University of California, San Diego}
\icmlaffiliation{mbzuai}{Mohamed bin Zayed University of Artificial Intelligence}


\icmlkeywords{Machine Learning, ICML}

\vskip 0.3in
]



\printAffiliationsAndNotice{\icmlEqualContribution} 

\begin{abstract}
\input{sections/0abstract}
\end{abstract}

\input{sections/1introduction}
\input{sections/2preliminaries}
\input{sections/3methodology}
\input{sections/4experiments}
\input{sections/5conclusion}
\input{sections/acknowledgements}
\input{sections/impact_statement}


\bibliography{bibliography}
\bibliographystyle{ms}

\newpage
\appendix
\onecolumn
\begin{center}
{\LARGE \bf 
Supplementary Material}
\end{center}
\input{sections/6appendices}

\end{document}

%% file: sections/0abstract.tex
Identifying latent variables and the causal structure involving them is essential across various scientific fields. While many existing works fall under the category of constraint-based methods (with e.g. conditional independence or rank deficiency tests), they may face empirical challenges such as testing-order dependency, error propagation, and choosing an appropriate significance level. These issues can potentially be mitigated by properly designed score-based methods, such as Greedy Equivalence Search (GES)~\citep{chickering2002optimal} in the specific setting without latent variables. Yet, formulating score-based methods with latent variables is highly challenging. In this work, we develop score-based methods that are capable of identifying causal structures containing causally-related latent variables with identifiability guarantees. Specifically, we show that a properly formulated scoring function can achieve score equivalence and consistency for structure learning of latent variable causal models. We further provide a characterization of the degrees of freedom for the marginal over the observed variables under multiple structural assumptions considered in the literature, and accordingly develop both exact and continuous score-based methods. This offers a unified view of several existing constraint-based methods with different structural assumptions. Experimental results validate the effectiveness of the proposed methods.

%% file: sections/1introduction.tex
\section{Introduction}\label{sec:intro}
At the core of understanding complex systems lies causal discovery, the identification of causal relations from observational data~\citep{spirtes2001causation,pearl2009causality}. One common assumption in causal discovery algorithms is the absence of latent confounders, known as \textit{causal sufficiency}, positing that the observed correlations stem either from true causation or can be sufficiently explained by other observed variables. Yet, real-world scenarios often defy this assumption. For instance, in psychological studies, the measured questionnaires are indirect proxies of latent mental factors. In unstructured data like images and texts, the observed pixels and words are confounded by latent semantic variables. Directly applying causal discovery methods without considering these latent variables can lead to false discoveries, as latent variables may introduce spurious correlations among observed ones that cannot be attributed to true causation.

Notable efforts have thus been made to identify the true causal relations in the presence of latent variables. Earliest attempts include Fast Causal Inference (FCI)~\citep{spirtes2001causation, zhang2008completeness} and its variants~\citep{colombo2012learning, spirtes2013causal, claassen2013learning, akbari2021recursive} that exploit conditional independence information. There are two main limitations of FCI: First, the results, presented by partial ancestral graphs (PAG)~\citep{richardson1996models}, tend to be overgeneralized -- e.g., whenever two observed variables may be confounded, it indicates so. Second, it focuses solely on causal relations among observed variables and does not provide information about those among latent variables. In short, FCI does not require specific assumptions about the latent structure, at the cost of having a less informative output. In contrast, one may often be interested in identifying the causal relations among latent variables (e.g., the latent mental and semantic variables in the above examples).\looseness=-1

Hence, another line of work has been developed to discover the causal structure also among latent variables. For the identifiability conditions, these methods typically introduce additional parametric assumptions to mitigate the large model indeterminacies faced by FCI. This includes rank or tetrad condition-based methods with linearity assumption~\citep{silva2003learning,silva2006learning,silva2005generalized,choi2011learning,kummerfeld2016causal,huang2022latent,dong2023versatile}, high-order moments-based methods~\citep{shimizu2009estimation,zhang2018causal,cai2019triad,salehkaleybar2020learning,xie2020generalized,adams2021identification,dai2022independence,chen2022identification,amendola2023third,wang2023causal}, matrix decomposition-based methods~\citep{anandkumar2013learning}, copula model-based methods~\citep{cui2018learning}, mixture oracles-based methods~\citep{kivva2021learning}, and multiple domains-based methods~\citep{zeng2021causal,sturma2023unpaired}. For the algorithmic procedures, these methods generally fall under the category of constraint-based methods, by matching the statistical properties to possible structural patterns and constructing the whole causal structure iteratively. A typical constraint-based method in the causally sufficient case is PC~\citep{spirtes1991pc}. Despite the asymptotic consistency, the empirical reliability of constraint-based methods may be limited due to \textit{testing-order dependency} and \textit{error propagation}~\citep{spirtes2010introduction,colombo2012learning}, especially when the number of variables is large.

\vspace{-0.1em}
To address such empirical issues of constraint-based methods, score-based causal discovery methods have been introduced, and may be more favored in practical applications~\citep{nandy2018high,ramsey2017million}. Unlike the iterative construction of a single causal graph by constraint-based methods, score-based methods assign a score to each potential graph reflecting how well it explains the observed data and generally search over the graph space to find the optimal graph. In the causally sufficient case, one typical score-based method is the Greedy Equivalence Search (GES)~\citep{chickering2002optimal}. There also exists several score-based methods that can handle latent variables~\citep{shpitser2012parameter,triantafillou2016score,nowzohour2017distributional,bhattacharya2020differentiable,shahin2020automatic,bernstein2020ordering,bellot2021deconfounded,claassen2022greedy}. Similar to FCI, most of them do not discover the causal relations among latent variables, except the method by \citet{zhang2004hierarchical} without identifiability guarantee. When latent variables are introduced and relations among them are further allowed in the causal structures, challenges arise in characterizing the degrees of freedom~\citep{geiger1996asymptotic,geiger2001stratified}, formulating a scoring function, and structuring the search procedure. We tackle these challenges in this paper, and to the best of our knowledge, this is the first score-based method that identifies causal structures containing causally-related latent variables with identifiability guarantees.

\vspace{-0.1em}
\textbf{Contributions.} \ \ 
We develop score-based methods, called SALAD (which stands for Score-bAsed Latent cAusal Discovery), for causal discovery of latent variable causal models, providing a unified view for several existing constraint-based methods \citep{silva2003learning,huang2022latent}. Our contributions can be summarized as follows:
\begin{itemize}
    \vspace{-0.4em}
    \item We develop a formulation of scoring function for identifying linear latent variable causal models. We show  (1) that it is score equivalent and (2) that minimizing it yields a structure that is algebraic equivalent to the true structure. The latter implies that both structures have the same equality constraints (on the marginal over the observed variables), including conditional independence and rank deficiency constraints.
    \item We provide a characterization of the degrees of freedom for the marginal over the observed variables under the structural assumptions considered by \citet{silva2003learning,huang2022latent}.
    \item We develop exact score-based methods for estimating the causal structure, and show that they can asymptotically identify the true equivalence class of the whole structure. We also provide continuous score-based methods in some of the settings to improve the computational efficiency.
    \item We demonstrate that the proposed score-based methods achieve improved performance over existing constraint-based methods for estimating the structures of latent variable causal models, which further validate the effectiveness of score-based methods.
\end{itemize}
\textbf{Notations.} \ \ 
For a matrix $M$, we define its support set as $\supp(M)\coloneqq\{(i,j):M_{i,j}\neq 0\}$. We denote by $M_{\set{S}, :}$ the rows in $M$ indexed by set $\set{S}$, and similarly by $M_{:, \set{S}}$ for the columns. For a directed acyclic graph (DAG) $\mathcal{G}$, we denote by $|\mathcal{G}|$ the number of edges in $\mathcal{G}$. Also, let $\diag(\mathbb{R}_{> 0}^m)$ be the set of $m\times m$ diagonal matrices with positive diagonal entries, $\mathbb{U}^{m}$ be the set of $m\times m$ strictly upper triangular matrices, and $\mathbb{G}^{m}$ be the set of graphs with $m$ measured variables that follow \cref{eq:linear_sem}. For set $\set{S}$, we define its $k$-partition as a partition of its elements into $k$ non-empty subsets.\looseness=-1

%% file: sections/2preliminaries.tex
\section{Latent Variable Causal Models}\label{sec:preliminaries}
In this section, we discuss several aspects of latent variable causal models. Specifically, we describe the preliminaries and problem setting in \cref{sec:problem_setting}, as well as the formulation of likelihood function in \cref{sec:likelihood}. We provide a discussion of latent variable causal models in \cref{app:latent_variable_models}.\looseness=-1

\subsection{Preliminaries and Problem Setting}\label{sec:problem_setting}
We consider a linear latent variable causal model with DAG $\mathcal{G}$, in which the measured variables $X=(X_1,\dots,X_m)$ and latent (unmeasured) variables $L=(L_1,\dots,L_n)$ follow the data generating procedure:\looseness=-1
\begin{equation}\label{eq:linear_sem}
L=CL+E_L \quad\text{and}\quad X=BL+E_X,
\end{equation}
where $E_X$ and $E_L$ are jointly independent noise terms that follow Gaussian distributions. The structure of DAG $\mathcal{G}$ is defined by the support of matrices $B$ and $C$, i.e., $L_j\rightarrow L_i$ is an edge in $\mathcal{G}$ if $C_{i,j}\neq 0$ and $L_j\rightarrow X_i$ is an edge in $\mathcal{G}$ if $B_{i,j}\neq 0$. For DAG $\mathcal{G}$, we denote by $B_\mathcal{G}\in\{0,1\}^{m\times n}$ the binary adjacency matrix that represent the edges from latent variables $L$ to measured variables $X$, and by $C_\mathcal{G}\in\{0,1\}^{n\times n}$ the binary adjacency matrix that represent the edges among latent variables $L$. Without loss of generality, we assume that matrices $C$ and  $C_\mathcal{G}$ are strictly upper triangular. 

Let $\Sigma_X$ and $\Sigma_L$ be the population covariance matrices of measured variables $X$ and latent variables $L$ respectively. Also let  $\Omega_X$ and $\Omega_L$ be the (diagonal) covariance matrices of noise terms $E_X$ and $E_L$ respectively. $\Sigma_L$ can be written as\looseness=-1
\[\Sigma_{L} =(I-C)^{-1}\Omega_L(I-C)^{-\top}.
\]
By $\Sigma_{X} =B \Sigma_{L} B^\top+\Omega_X$, we then have
\begin{equation}\label{eq:measured_covariance_matrix}
\Sigma_{X}=B(I-C)^{-1}\Omega_L(I-C)^{-\top} B^{\top}+\Omega_X.
\end{equation}
We say that a DAG $\mathcal{G}$ can generate a covariance matrix if there exists a parameterization of $\mathcal{G}$ such that \cref{eq:measured_covariance_matrix} holds. Furthermore, since the labeling of latent variables in general cannot be identified, we say that two DAGs are Markov equivalent if they are Markov equivalent after relabeling of latent variables. Given $T$ i.i.d. samples of variables $X$, denoted as $\mathbf{D}$ with empirical covariance matrix $\EmpCov$, the goal is to estimate the structure $\mathcal{G}$ up to certain type of model equivalence (specified in \cref{sec:1_factor_model,sec:hierarchical_structures}). 

\subsection{Formulation of Likelihood Function}\label{sec:likelihood}
We first discuss about the indeterminacy of parameter $\Omega_L$ via the following lemma, since it affects how we formulate the likelihood. The proof is given in \cref{app:proof_indeterminacy_omegaL}.
\begin{restatable}[Indeterminacy of $\Omega_L$]{lemma}{IndeterminacyOfOmegaL}\label{lemma:indeterminacy_Omega_L}
For any parameters $B, C, \Omega_X, \Omega_L$, and $\Sigma_X$ that follow \cref{eq:measured_covariance_matrix}, there exist parameters $\tilde{B}$ and $\tilde{C}$ with $\supp(B)=\supp(\tilde{B})$ and $\supp(C)=\supp(\tilde{C})$ such that
\[
\Sigma_{X}=\tilde{B}(I-\tilde{C})^{-1}(I-\tilde{C})^{-\top} \tilde{B}^{\top}+\Omega_X.
\]
\end{restatable}
In other words, any covariance matrix $\Sigma_X$ resulting from DAG $\mathcal{G}$ and arbitrary $\Omega_L$ can be generated by alternative parameters from the same DAG with $\tilde{\Omega}_L=I$. This implies that the parameter $\Omega_L$ cannot be estimated from $\Sigma_X$ without additional information and further assumption. Furthermore, since the goal is to estimate the structure $\mathcal{G}$, this suggests that one may assume $\Omega_L$ to be an identity matrix during estimation without loss of generality. It is worth noting that such indeterminacy of $\Omega_L$ has been discussed in various existing works \citep{squires2023linear}, which we make precise here, as it is crucial for formulating the likelihood.

We now provide the likelihood formulation for the linear latent variable causal model in \cref{eq:linear_sem}. As suggested by \cref{lemma:indeterminacy_Omega_L}, we set $\Omega_L=I$ in the likelihood. Given the empirical covariance matrix $S$ obtained from $T$ samples, the negative log-likelihood is given up to additive constant by
\begin{flalign*}
& \mathcal{L}\left(B, C, \Omega_X; \mathbf{D}\right)\\
& \quad = \frac{T}{2}\tr\left(\EmpCov\left(B(I-C)^{-1}(I-C)^{-\top} B^{\top}+\Omega_X\right)^{-1}\right) \\
& \quad \quad \quad +\frac{T}{2} \log \operatorname{det}\left(B(I-C)^{-1}(I-C)^{-\top} B^{\top}+\Omega_X\right).
\end{flalign*}

%% file: sections/3methodology.tex
\section{Score-Based Identification of Latent Variable Causal Models}\label{sec:score_based_latent_causal_discovery}
In this section, we discuss how to learn linear latent variable causal models with scoring function. First, we introduce the notion of distribution sets and equality constraints in \cref{sec:distribution_sets}. We formulate the scoring function in \cref{sec:score}, and show how it enables structure identification up to algebraic equivalence in \cref{sec:score_identification}. We then discuss about the BIC score in \cref{sec:bic_score}.

\subsection{Distribution Sets and Equality Constraints}\label{sec:distribution_sets}
We describe the notion of distribution set that is a key ingredient of our score-based search procedure. It refers to the set of marginal distributions generated by a specific structure.
\begin{definition}[Distribution set]
The distribution set of DAG~$\mathcal{G}$, denoted by $\mathcal{M}(\mathcal{G})$, is defined as
\begin{flalign*}
\mathcal{M}(\mathcal{G})\coloneqq\{& B(I-C)^{-1}\Omega_L(I-C)^{-\top} B^{\top}+\Omega_X:\\
&\, \supp(B)\subseteq\supp(B_\mathcal{G}),\supp(C)\subseteq\supp(C_\mathcal{G}),\\
&\, \Omega_X\in\diag(\mathbb{R}_{> 0}^m),\Omega_L\in\diag(\mathbb{R}_{> 0}^n)\}.
\end{flalign*}
\end{definition}
Specifically, $\mathcal{M}(\mathcal{G})$ is the set of covariances matrices $\Sigma_X$ that can be generated by DAG $\mathcal{G}$ by varying the parameters in matrices $B$, $C$, $\Omega_X$, and $\Omega_L$. Moreover, since $C_\mathcal{G}$ is acyclic by assumption, we have $(I-C)^{-1}=\sum_{k=0}^{n-1} C^k$. It follows that \cref{eq:measured_covariance_matrix} is a polynomial map, and thus the distribution set $\mathcal{M}(\mathcal{G})$ is, by Tarski–Seidenberg theorem~\citep{benedetti1990real}, a semialgebraic set. Note that a set is said to be \emph{semialgebraic} if it can be equivalently represented by a finite number of polynomial equalities and inequalities~\citep{benedetti1990real}.

Structure $\mathcal{G}$ imposes various types of equality (i.e., algebraic) constraints on the covariance matrices, such as conditional independence (i.e., vanishing partial correlation) constraints \citep{spirtes2001causation}, rank deficiency (i.e., vanishing determinant) constraints \citep{spirtes2001causation,sullivant2010trek}, and possibly Verma constraints~\citep{verma1991equivalence}. We refer the readers to \citet{drton2018algebraic} for an overview. Let $H(\mathcal{G})$ be the set of equality constraints imposed by structure $\mathcal{G}$ on the distribution set $\mathcal{M}(\mathcal{G})$, and $\mathbb{H}^m\coloneqq\bigcup_{\mathcal{G}\in\mathbb{G}^m}H(\mathcal{G})$ be the set of possible equality constraints imposed by any structure $\mathcal{G}$ (with $m$ measured variables). Two structures $\mathcal{G}_1$ and $\mathcal{G}_2$ are said to be \emph{algebraic equivalent} if they lead to the same equality constraints, i.e., $H(\mathcal{G}_1)=H(\mathcal{G}_2)$~\citep{ommen2017algebraic}.\footnote{In the terminology of algebraic geometry, $\mathcal{M}(\mathcal{G}_1)$ and $\mathcal{M}(\mathcal{G}_2)$ share the same vanishing ideal or Zariski closure \citep{cox2008ideals}.}

Furthermore, let $\dim(\mathcal{G})$ denote the model dimension or degrees of freedom of DAG $\mathcal{G}$ for the marginal over the observed variables, which can be viewed as the number of free parameters for the distribution set $\mathcal{M}(\mathcal{G})$. In general, the degrees of freedom are not necessarily equal to the number of parameters (i.e., sum of number of edges and measured variables) in the presence of latent variables \citep{geiger1996asymptotic,geiger2001stratified}. In \cref{sec:1_factor_model,sec:hierarchical_structures}, we further characterize the degrees of freedom under specific structural assumptions.

\input{sections/3score}

\subsection{Identifying Structures up to Algebraic Equivalence}\label{sec:score_identification}
Having formulated the scoring function in \cref{sec:score}, the question remains as how to leverage it to identify the underlying structure $\mathcal{G}$. To do so, a key ingredient is to establish the correspondence between the covariance matrix and the structure $\mathcal{G}$. As discussed in \cref{sec:distribution_sets}, the structure $\mathcal{G}$ imposes different types of constraints on the entries of covariance matrices, including equality and inequality constraints. Here, we adopt the following assumption which requires that the equality constraints are imposed by the structure $\mathcal{G}$.
\begin{assumption}[Generalized faithfulness \citep{ghassami2020characterizing}]\label{assumption:generalized_faithfulness}
A distribution $\Sigma_X$ is said to be generalized faithful to DAG $\mathcal{G}$ if the entries of $\Sigma_X$ satisfy an equality constraint $\kappa\in\mathbb{H}^m$ only if $\kappa\in H(\mathcal{G})$.
\end{assumption}
It is worth noting that different types of faithfulness assumptions have been adopted in causal discovery~\citep{spirtes2001causation,ghassami2020characterizing,huang2022latent} to relate the constraints of the distributions (e.g., conditional independence and rank deficiency constraints) to the underlying structure. This is often motivated by the fact that the set of parameters violating these assumptions has Lebesgue measure zero (see, e.g., \citet[Proposition~8]{ghassami2020characterizing}).

We then present the following result that describes the notion of equivalence achieved by minimizing the scoring function. The proof is provided in \cref{app:thm_equivalence_equality_constraints}, which is partly inspired by the proof of \citet[Theorem~3]{ghassami2020characterizing}.
\vspace{-0.7em}
\begin{restatable}[Algebraic equivalence]{theorem}{TheoremEquivalenceEqualityConstraints}\label{thm:equivalence_equality_constraints}
Suppose the true DAG $\mathcal{G}^*$ and the distribution $\Sigma_X$ satisfy the generalized faithfulness assumption. Let $\hat{\mathcal{G}}\in \argmin_{\mathcal{G}\in\mathbb{G}^m}\score_{\textrm{dim}}(\mathcal{G},\mathbf{D})$. Then, $\hat{\mathcal{G}}$ and $\mathcal{G}^*$ are algebraic equivalent, i.e., $H(\hat{\mathcal{G}})=H(\mathcal{G}^*)$, in the large sample limit.
\end{restatable}
\vspace{0.2em}
\begin{remark}
Under generalized faithfulness assumption, \cref{thm:equivalence_equality_constraints} implies that minimizing the scoring function leads to a structure with the same equality constraints (on the marginal over the measured variables) as the true structure.
\end{remark}
\vspace{-0.1em}
In general, relating the estimated structure to the true one, which are algebraic equivalent, can be challenging without any restrictions on the structures. In \cref{sec:1_factor_model,sec:hierarchical_structures}, we show that, under specific structural assumptions, \cref{thm:equivalence_equality_constraints} helps achieve notions of model equivalence that are more fine-grained than algebraic equivalence (including Markov equivalence in \cref{sec:1_factor_model}). Therefore, a general recipe may involve identifying suitable structural assumptions that allow algebraic equivalence to translate into more fine-grained notions of model equivalence. This enables the application of the score based procedure in \cref{thm:equivalence_equality_constraints}, given an appropriate characterization of the degrees of freedom. We give a further discussion of generalized faithfulness and algebraic equivalence in \cref{app:generalized_faithfulness,app:algebraic_equivalence}, respectively.

\vspace{-0.1em}
\subsection{Remark on the BIC Score}\label{sec:bic_score}
The scoring function discussed in \cref{sec:score} is justified in the large sample limit and may not perform well for finite-sample cases. We consider the BIC score~\citep{schwarz1978estimating,chickering2002optimal} that maximizes the likelihood while penalizing the degrees of freedom of structure $\mathcal{G}$:
\[
\score_{\textrm{BIC}}(\mathcal{G},\mathbf{D}) \coloneqq \score_{\mathcal{L}}(\mathcal{G},\mathbf{D})+\frac{\log T}{2} \dim(\mathcal{G}).
\]
where $\score_{\mathcal{L}}(\mathcal{G},\mathbf{D})$ denotes the optimal negative log-likelihood w.r.t. structure $\mathcal{G}$, given by
\begin{equation}\label{eq:score_likelihood}
\score_{\mathcal{L}}(\mathcal{G},\mathbf{D})\coloneqq\min_{\substack{(B,C,\Omega_X):\\ \supp(B)\subseteq \supp(B_\mathcal{G}),\\ \supp(C)\subseteq \supp(C_\mathcal{G}),\\ \Omega_X\in\diag(\mathbb{R}_{> 0}^m)}} \mathcal{L}\left(B, C, \Omega_X; \mathbf{D}\right).
\end{equation}
Since it may not be straightforward to derive a closed-form solution, various numerical solvers or continuous optimization methods, such as  L-BFGS \citep{byrd1995limited} and gradient descent, as well as the expectation-maximization algorithm~\citep{dempster1977maximum}, can be used to compute the maximum likelihood above.

It is worth noting that the BIC score has been widely adopted in score-based causal discovery \citep{chickering2002optimal}. \citet{haughton1988choice} showed that it is an asymptotic approximation for the log marginal likelihood of \emph{curved} exponential families, which include Gaussian DAG models without latent variables \citep{geiger2001stratified,richardson2002ancestral}. In the presence of latent variables, the models are \emph{stratified} exponential families, and complications arise in using BIC for model selection. Although the typical theoretical justifications of using BIC \citep{schwarz1978estimating,haughton1988choice} may not apply for identifying latent variable causal models in our setting, we apply it in place of $\score_{\textrm{dim}}(\mathcal{G},\mathbf{D})$ in our experiments, since the latter is justified in the large sample limit and may not perform well for finite-sample cases. Surprisingly, using the BIC score leads to a superior empirical performance, specifically under the structural assumptions described in \cref{sec:1_factor_model,sec:hierarchical_structures}. This suggests that BIC may be a valid scoring criterion in these cases. Therefore, future works involve studying the theoretical justifications of using BIC score under these structural assumptions.\looseness=-1

Recall that a scoring function is \emph{score equivalent} if every pair of Markov equivalent structures have the same score~\citep{chickering2002optimal}. This is a desirable property for score-based procedure as it implies that we can search in the space of Markov equivalence classes (MECs) instead of DAGs. That is, one does not have to compute the score multiple times for the DAGs in the same MEC, which may help improve the runtime. We show that our scoring functions satisfy such a property, with a proof given in \cref{app:proof_proposition_score_equivalence}.
\vspace{-0.55em}
\begin{restatable}[Score equivalence]{proposition}{PropositionScoreEquivalence}\label{proposition:score_equivalence}
Suppose that DAGs $\mathcal{G}_1$ and $\mathcal{G}_2$ are Markov equivalent. Then, we have $\score_{\textrm{dim}}(\mathcal{G}_1,\mathbf{D})=\score_{\textrm{dim}}(\mathcal{G}_2,\mathbf{D})$ and $\score_{\textrm{BIC}}(\mathcal{G}_1,\mathbf{D})=\score_{\textrm{BIC}}(\mathcal{G}_2,\mathbf{D})$.
\end{restatable}

\vspace{-0.25em}
\section{Linear 1-Factor Latent Variable Models}\label{sec:1_factor_model}
In the previous section, we show that the scoring function can produce a structure algebraic equivalent to the ground truth. We now discuss how such result helps estimate a structure up to Markov equivalence. In this section, we focus on the structural assumption by \citet{silva2003learning,silva2006learning}.\footnote{In our setting, it may be sufficient to require that each latent variable has at least two measured variables as children.}\looseness=-1
\begin{assumption}[{{\citet{silva2003learning}}}]\label{assumption:silva_graphical_criterion}
Each measured variable has a single latent parent, and each latent variable has at least three measured variables as children.
\end{assumption}
An example illustrating the above assumption is provided in \cref{fig:silva_illustration}. \citet{silva2003learning} proposed a search procedure based on statistical tests of tetrad constraints that can identify structures under this assumption. In this section, we develop a score-based method based on this structural assumption. We first characterize the degrees of freedom of the structure in \cref{sec:1_factor_model_dimension}, as required by the scoring function. We then establish the consistency and provide an exact score-based search procedure in \cref{sec:1_factor_model_exact_search}. We also develop a continuous search procedure in \cref{sec:1_factor_model_continuous_search} that may be more computationally efficient.

\begin{figure}[!t]
\centering
\includegraphics[width=0.49\textwidth]{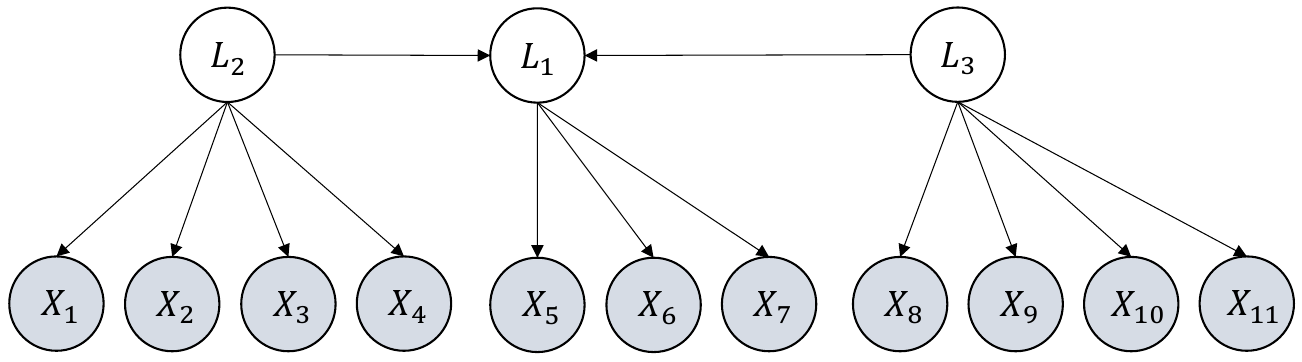}
\caption{Example of 1-factor latent variable model.}
\label{fig:silva_illustration}
\end{figure}

\subsection{Degrees of Freedom}\label{sec:1_factor_model_dimension}
The scoring function requires a proper specification of the degrees of freedom during the search procedure. For the structural assumption in \cref{assumption:silva_graphical_criterion}, the degrees of freedom, as one may expect, equals the number of edges in DAG $\mathcal{G}$ plus the number of measured variables. Here, the number of edges include those among the latent variables and those from the latent variables to the measured ones. The proof follows straightforwardly from parameter identifiability under \cref{assumption:silva_graphical_criterion} \citep{bollen1989general}, which is provided in \cref{app:proposition_dimension_silva} for completeness.
\begin{restatable}[Degrees of freedom]{proposition}{PropositionDimensionSilva}\label{proposition:dimension_silva}
Suppose that DAG $\mathcal{G}$ satisfies  \cref{assumption:silva_graphical_criterion}. Then, $\operatorname{dim}(\mathcal{G})=|\mathcal{G}|+m$.
\end{restatable}
To illustrate, the degrees of freedom of the example in \cref{fig:silva_illustration} are simply equal to $24$.
The above property holds in many other settings such as the typical setting without latent confounders \citep{chickering2002optimal}, as well as those with bow-free acyclic mixed graphs \citep{brito2002identification} and cycles (excluding 2-cycles) \citep{amendola2020structure}. Note that such property, while desirable, does not hold in general for structures with latent variables \citep{geiger1996asymptotic}. For instance, the degrees of freedom of the structures that we consider in \cref{sec:hierarchical_structures}  are generally not equal to $|\mathcal{G}|+m$.

\subsection{Consistency and Exact Score-Based Search}\label{sec:1_factor_model_exact_search}
Having characterized the degrees of freedom, we now establish the correctness of score-based approach under \cref{assumption:silva_graphical_criterion} and accordingly develop an exact search procedure. Specifically, under \cref{assumption:silva_graphical_criterion} and the generalized faithfulness assumption, we show that the structure with the optimal score is Markov equivalent to the true structure.
\begin{restatable}[Correctness]{theorem}{TheoremCorrectnessSilvaExact}\label{thm:correctness_silva_exact}
Suppose that the true DAG $\mathcal{G}^*$ and the distribution $\Sigma_X$ satisfy the generalized faithfulness assumption, and that $\mathcal{G}^*$ satisfies \cref{assumption:silva_graphical_criterion}. Let $\hat{\mathcal{G}}$ be a global minimizer of the following optimization problem:
\begin{equation}\label{eq:optimization_silva}
\begin{aligned}
\min_{\mathcal{G}\in\mathbb{G}^m}\quad & \score_{\textrm{dim}}(\mathcal{G},\mathbf{D}) \\
\subjectto \quad & \mathcal{G} \text{ satisfies  \cref{assumption:silva_graphical_criterion}},
\end{aligned}
\end{equation}
where $\dim(\mathcal{G})=|\mathcal{G}|+m$. Then, $\hat{\mathcal{G}}$ and $\mathcal{G}^*$ are Markov equivalent in the large sample limit.
\end{restatable}
The proof can be found in \cref{app:thm_correctness_silva_exact}, which leverages \cref{thm:equivalence_equality_constraints} that shows how the scoring function produces a structure that is algebraic equivalent to the true structure. Moreover, the theorem above indicates that one could perform exact search for all structures under \cref{assumption:silva_graphical_criterion}. A naive approach is to iterate over all possible structures and check if each of them satisfy \cref{assumption:silva_graphical_criterion}. This may be computationally infeasible because much of the time may be spent on structures that do not fall within the model class.

The question is then how to efficiently enumerate and perform exact search for these structures. We provide an algorithm to do so in \cref{alg:enumerate_silva_structures}. Leveraging the score equivalence property in \cref{proposition:score_equivalence}, we consider only structures that are not Markov equivalent to one another, since they are indistinguishable based on \cref{thm:correctness_silva_exact} and give rise to the same score. First, we generate the possible structures $C_\mathcal{G}$ among the latent variables that are not Markov equivalent to one another. To construct the structure $B_\mathcal{G}$ from latent variables to measured variables, we then find all ordered partitions of measured variables and add each subset from the partition to be the children of each latent variable. We compute the score for each structure enumerated by \cref{alg:enumerate_silva_structures}, and find the structure with the optimal score. Under \cref{thm:correctness_silva_exact}, such an exact search procedure will output a structure Markov equivalent to the true one.

\begin{algorithm}[tb]
   \caption{Enumerating structures under \cref{assumption:silva_graphical_criterion}}
   \label{alg:enumerate_silva_structures}
\begin{algorithmic}
   \STATE {\bfseries Input:} Measured variables $X_1,\dots,X_m$
   \STATE {\bfseries Output:} Set of DAGs $\set{A}$
   \STATE Initialize $\set{A}$ as an empty set;
   \FOR{$n=1$ {\bfseries to} $\lfloor m/3\rfloor$}
   \FOREACH{latent MEC with $n$ variables}
   \STATE Generate a latent DAG $C_\mathcal{G}$ from the latent MEC;
   \FOREACH{ordered $n$-partition $(\set{P}_j)_{j=1}^n$ of $\{X_i\}_{i=1}^m$}
   \IF{$|\set{P}_j|\geq 3$ for $j=1,\dots,n$}
   \STATE Construct DAG $\mathcal{G}$ with latent DAG $C_\mathcal{G}$ and each latent $L_j$ pointing to the variables in $\set{P}_j$;
   \IF{$\mathcal{G}$ is not Markov equivalent to all DAGs in $\set{A}$}
   \STATE $\set{A}\leftarrow\set{A}\cup\{\mathcal{G}\}$;
   \ENDIF
   \ENDIF
   \ENDFOR
   \ENDFOR
   \ENDFOR
   \STATE {\bfseries return} set $\set{A}$
\end{algorithmic}
\end{algorithm}

\subsection{Continuous Search}\label{sec:1_factor_model_continuous_search}
The exact search procedure presented in the previous section requires computing the score for each structure satisfying the structural assumption, which can be computationally intensive when there is a large number of variables. For instance, when using the BIC score, each computation involves solving a continuous optimization problem in \cref{eq:score_likelihood}; the same applies to $\score_{\textrm{dim}}(\mathcal{G},\mathbf{D})$. This is inherent to discrete score-based search that assigns a score to each structure. A key question naturally arises: how do we unify the structure search part and likelihood computation into a single continuous optimization problem? Such a unified procedure helps reduce the computational burden of separately computing the score for each structure in a discrete search. Furthermore, this aligns with recent studies in continuous optimization for causal discovery~\citep{zheng2018notears,ng2020role,vowels2022dags}.

We first provide a reformulation of \cref{eq:optimization_silva} with the BIC score that is more amenable to continuous optimization. The key lies in characterizing the penalty term $|\mathcal{G}|$ and the constraint involving \cref{assumption:silva_graphical_criterion}. Specifically, we solve the following constrained optimization problem:
\begin{align}
\min_{\substack{M_B\in\{0,1\}^{m\times \bar{n}},\nonumber \\
M_C\in\{0,1\}^{\bar{n}\times \bar{n}},\nonumber \\
B\in\mathbb{R}^{m\times \bar{n}},C\in\mathbb{U}^{\bar{n}},\nonumber\\
\Omega_X\in\diag(\mathbb{R}_{> 0}^m)}}\quad & \Big(\frac{1}{T}\mathcal{L}\left(M_B\odot B, M_C\odot C, \Omega_X; \mathbf{D}\right)\nonumber\\[-7ex]
&\qquad +\lambda \|M_B\|_1+\lambda \|M_C\|_1 \Big)\nonumber\\[2ex]
\subjectto \quad & \sum_{k=1}^{\bar{n}} (M_B)_{i,k}-1=0 ,\; i\in[m],\label{eq:continuous_optimization_silva}\\
& \Bigg(\bigg(\sum_{k=1}^m (M_B)_{k,j}+\sum_{k=1}^{\bar{n}} (M_C)_{k,j}\bigg)\nonumber\\
& \qquad \bigg(\sum_{k=1}^m (M_B)_{k,j}-3\bigg)\Bigg)\geq 0,\; j\in[\bar{n}],\nonumber
\end{align}
where $\bar{n}=\lfloor m/3\rfloor$ is an upper bound of the number of latent variables and $\lambda=\log T/2T$. In the formulation above, the matrices $M_B$ and $M_C$ can be viewed as the support matrices of $B$ and $C$; they act as binary masks which indicate which edges are present in the structure. Furthermore, the two constraints in \cref{eq:continuous_optimization_silva} serve as a characterization of \cref{assumption:silva_graphical_criterion} using the support matrices $M_B$ and $M_C$. Specifically, the first constraint requires each row of $M_B$ to have one nonzero entry (i.e., each measured variable has one a single latent parent). The second constraint requires each column of $M_B$ and $M_C$ to satisfy the following: either the column of $M_B$ has at least three nonzero entries, or the column of $M_B$ and $M_C$ have no nonzero entries (i.e., each latent variable has at least $3$ measured variables as children, or no child at all).

We now discuss how to solve \cref{eq:continuous_optimization_silva} using continuous constrained optimization procedure. We first introduce slack variable $t_i\geq 0$ and convert the inequality constraints into equality constraints. To estimate the binary matrices $M_B$ and $M_C$, we apply the Gumbel-Softmax technique~\citep{maddison2017concrete,jang2017gumbelsoftmax} that is widely used to sample and approximate samples from a categorical distribution, which has also been adopted in continuous optimization approaches for causal discovery \citep{ng2022masked,brouillard2020differentiable}. Specifiaclly, we apply Gumbel-Sigmoid for each entry of $M_C$, and Gumbel-Softmax with $\bar{n}$ categories for each row of $M_B$; the latter also incorporates the first constraint of \cref{eq:continuous_optimization_silva} that requires each row of $M_B$ to have one nonzero entry. The resulting continuous constrained optimization problem can then be solved using standard methods such as augmented Lagrangian method and quadratic penalty method \citep{bertsekas1982constrained,bertsekas1999nonlinear,nocedal2006numerical}. These methods transform the constrained problem into a series of unconstrained problems, each of which can be solved via continuous optimization methods such as gradient descent or L-BFGS \citep{byrd1995limited}. In this work, we adopt augmented Lagragian method that is commonly used in causal discovery with continuous optimization \citep{zheng2018notears,vowels2022dags}. 

\vspace{-0.55em}
\section{Linear Latent Hierarchical Structures}\label{sec:hierarchical_structures}
\vspace{-0.15em}
The structural assumption in \cref{sec:1_factor_model} requires (i) each measured variable to have only one latent parent and (ii)~each latent variable to have measured children. In real-world cases, the structure may be more complex -- the measurement model may not be a tree and some latent variables may not have measured children. Thus, we also consider a more general assumption formulated by \citet{huang2022latent}.

We use $\mathbf{L}$ to denote a set of latent variables and write $\children(\set{L})\coloneqq\bigcup_{L_i \in \set{L}} \children(L_i)$,  $\parents(\set{L})\coloneqq\bigcup_{L_i \in \set{L}} \parents(L_i)$, and $\descendant(\set{L})\coloneqq\bigcup_{L_i \in \set{L}} \descendant(L_i)$, as their parents, children, and descendants, respectively. We now explain the notions of pure children and latent atomic cover, which serve as the fundamental building blocks of the whole structure.

\begin{definition} [Pure children set \citep{huang2022latent}]
Let $\purechildren(\set{L})$ denote the collection of pure-child sets of $\set{L}$.
We say that $\set{V}$ is a pure-child set of $\set{L}$, written
$\set{V} \in \purechildren(\set{L})$, if (i)
$\set{V} \cap \set{L} = \emptyset$, (ii) $\set{V} \subseteq \children(\set{L})$,
(iii) $\parents(\set{V}) = \set{L}$, and (iv)
$\descendant(\set{V}) \cap \set{L} = \emptyset$.
\label{definition:pch}
\end{definition}

\begin{definition} [Latent atomic cover \citep{huang2022latent}]
  Let $\set{L}$ be a set of latent variables in $\graph$ with $|\set{L}|=k$.
  $\set{L}$ is a latent atomic cover 
    if the following conditions hold:
    \vspace{-1.1em}
  \begin{itemize}
      \item [(i)] There exists 
      a set of variables $\set{C}$  such that
      $\set{C} \in \purechildren(\set{L})$
      and 
     $|\set{C}|\geq k+1$.
     \vspace{-0.5em}
          \item [(ii)] 
    If $\node{V}\in \parents(\mathbf{L})$, then $\node{V}$ is also the parent of all elements in $\set{L}$. Plus, there exists a  set $\set{N}$ with $|\set{N}|\geq k+1$  
    such that $\set{N}\cap\set{L}=\emptyset$,  every element in $\set{N}$ is a neighbor of every element in $\set{L}$,
    and $\set{L}$ d-separates $\set{N}$ and $\set{C}$.
   \vspace{-0.5em}
\item [(iii)] There does not exist a partition of $\set{L}$
such that all elements in the partition are latent atomic covers.
\vspace{-0.5em}
  \end{itemize}
  \vspace{-0.5em}
\label{definition:lac}
\end{definition}
Having introduced the required notions, we now provide the structural assumptions considered by \citet{huang2022latent}.
\begin{assumption}[{Identifiable linear latent hierarchical graph~\citep{huang2022latent}}]
\label{assumption:llh_graphical_criterion}
A graph $\graph$ is an identifiable 
linear latent hierarchical graph if (i)
every latent variable $L_i$ belongs to at least one latent atomic cover and there is no triangle structure in the graph,
and (ii)
if there exists a set of variables $\set{V}$
such that every variable in $\set{V}$ is a collider of two latent atomic covers $\set{L}_1$, $\set{L}_2$, and denote by $\set{T}$ the minimal set of variables that d-separates $\set{L_1}$ from $\set{L_2}$, then we must have 
$|\set{V}| + |\set{T}| \geq |\set{L}_1|+|\set{L}_2|$.
\end{assumption}

The assumption above requires that each latent variable belongs to at least one latent atomic cover, since a latent atomic cover is the minimal identifiable substructure in a graph using rank constraints of covariance over observed variables. Also, the assumption requires certain graphical patterns that are related to the common children across different latent atomic covers for the identifiability of the whole latent structure. The assumption above may be more general than Assumption \ref{assumption:silva_graphical_criterion} as it allows each measured variable to have multiple latent variables together as parents, and also allow some latent variables to not have any measured child at all; see \cref{app:structural_assumptions} for more details.  An example is given in \cref{fig:llh_illustration}. Under the assumption above, \citet{huang2022latent} developed a constraint-based method based on rank deficiency test to estimate the equivalence class of the true structure.\looseness=-1

In this section, we develop a score-based method under \cref{assumption:llh_graphical_criterion}. We characterize the degrees of freedom in \cref{sec:hierarchical_structures_dimension} and provide an exact search method in \cref{sec:hierarchical_structures_exact_search}. We do not provide a continuous search method for this structural assumption, since it cannot be straightforwardly formulated as inequality constraints, similar to \cref{eq:continuous_optimization_silva}.\looseness=-1

\begin{figure}[!t]
\centering
\includegraphics[width=0.44\textwidth]{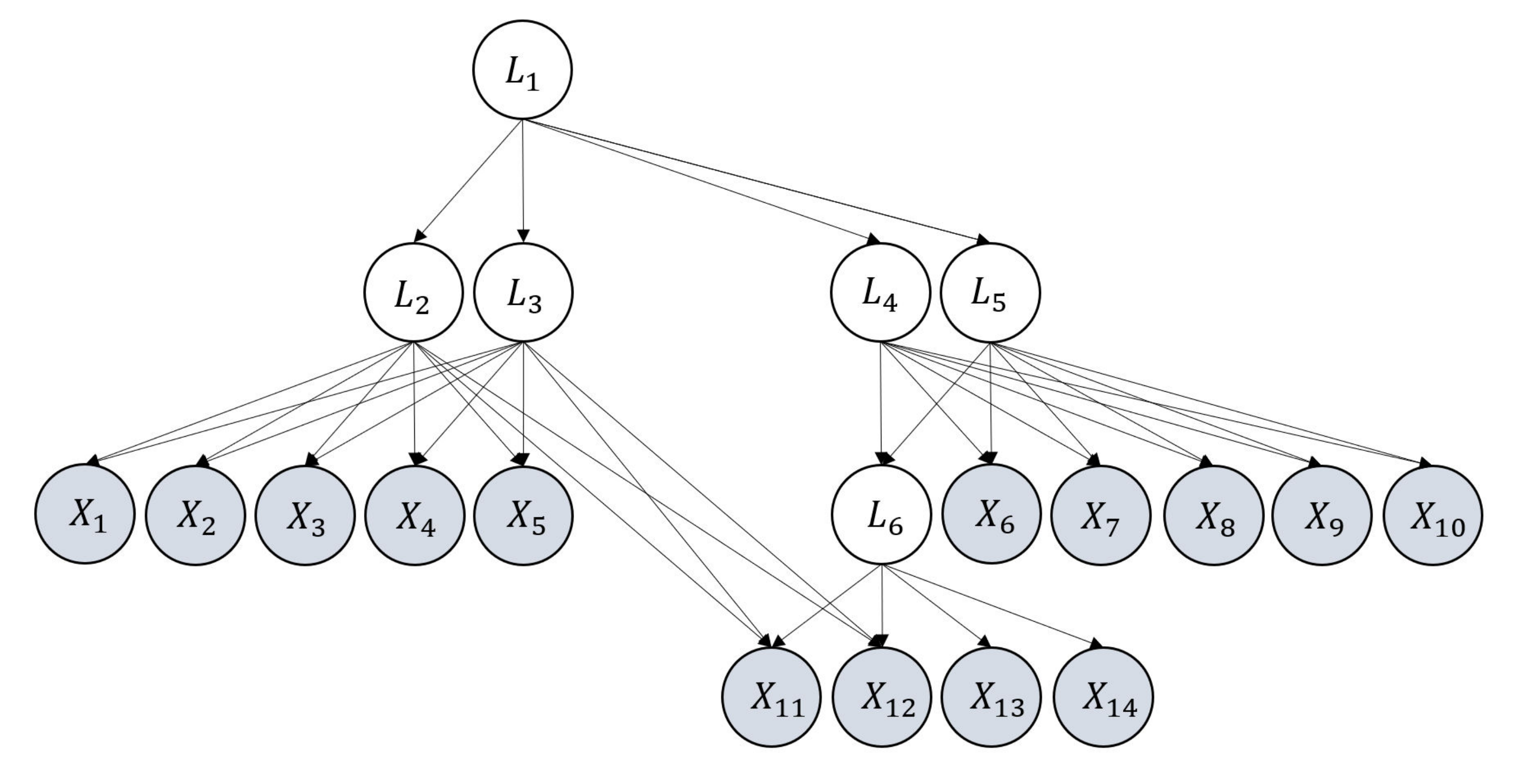}
\caption{Example of latent hierarchical structure.}
\label{fig:llh_illustration}
\end{figure}

\subsection{Degrees of Freedom}\label{sec:hierarchical_structures_dimension}
As noted in \cref{sec:score,sec:1_factor_model}, a key ingredient of the score-based method is the specification of the degrees of freedom. It may be natural to expect that the degrees of freedom are equal to the number of edges and measured variables, similar to the structural assumption considered in \cref{sec:1_factor_model} and the standard setting without latent variables~\citep{chickering2002optimal}. However, this property does not hold for latent hierarchical structures, as illustrated by the following lemma.
\begin{restatable}{proposition}{PropositionReducedDimension}\label{proposition:reduced_dimension}
Suppose that DAG $\mathcal{G}$ follows the linear latent variable causal model in \cref{eq:linear_sem}. Suppose also that there exist $k\geq 2$ latent variables in $\mathcal{G}$ with the same set of parents and children, where either the number of parents or children is at least $k$. Then, $\dim(\mathcal{G})\leq |\mathcal{G}|+m-k(k-1)/2$.
\end{restatable}
The proof is given in \cref{app:proposition_reduced_dimension}. As shown in the proof, under these circumstances, there exists an alternative structure $\tilde{\mathcal{G}}$ obtained by removing $k(k-1)/2$ edges (corresponding to edges involving the parents or children) from $\mathcal{G}$ such that $\tilde{\mathcal{G}}$ leads to the same distribution set as $\mathcal{G}$, i.e., $\mathcal{M}(\tilde{\mathcal{G}})=\mathcal{M}(\mathcal{G})$. Thus, the degrees of freedom in this scenario are reduced compared to the number of edges and measured variables. This reduction can be intuitively explained by the redundancy of certain edges in such situations. For instance, the degrees of freedom for the structure are $44$ instead of $46$ (i.e., the sum of number of edges and measured variables), because the variables $\{L_2,L_3\}$ and $\{L_4,L_5\}$ (i.e., in the same atomic covers) have the same parents and children.\looseness=-1

\begin{algorithm}[tb]
   \caption{Degrees of freedom under \cref{assumption:llh_graphical_criterion}}
   \label{alg:count_llh_dimension}
\begin{algorithmic}
   \STATE {\bfseries Input:} Structure $\mathcal{G}$
   \STATE {\bfseries Output:} Degrees of freedom $d$
   \STATE Initialize degrees of freedom $d\leftarrow |\mathcal{G}|+m$;
   \FOREACH{subset $\mathbf{L}$ of latent variables where $|\mathbf{L}|\geq 2$}
   \IF{variables in $\mathbf{L}$ have the same parents and children {\bfseries and} no proper superset of $\mathbf{L}$
    satisfies the previous condition}
   \STATE $d\leftarrow d - |\mathbf{L}|(|\mathbf{L}|-1)/2$;
   \ENDIF
   \ENDFOR
   \STATE {\bfseries return} degrees of freedom $d$
\end{algorithmic}
\end{algorithm}

Building on the result above, we develop a procedure in \cref{alg:count_llh_dimension} to calculate the degrees of freedom under \cref{assumption:llh_graphical_criterion}. Specifically, the algorithm iterates over all subsets of latent variables and calculate the degrees of freedom that can be reduced. The following proposition shows that the algorithm outputs the upper bound of the degrees of freedom, with a proof provided in \cref{app:proposition_dimension_llh}.
\vspace{-0.1em}
\begin{restatable}[Degrees of freedom]{proposition}{PropositionDimensionLLH}\label{proposition:dimension_llh}
Suppose that DAG $\mathcal{G}$ satisfies  \cref{assumption:llh_graphical_criterion}. Then, \cref{alg:count_llh_dimension} outputs the upper bound of  $\operatorname{dim}(\mathcal{G})$.
\end{restatable}
\vspace{-0.1em}
We conjecture, supported by simulations over $10,000$ examples (by computing the rank of Jacobian matrices \citep{geiger2001stratified}) and the experiments in \cref{sec:experiments}, that the upper bound provided by this algorithm is tight under operator $\mathcal{O}_{\textrm{min}}(\mathcal{O}_{\textrm{skeleton}}(\cdot))$, although a proof seems to involve tools from algebraic statistics and is not straightforward. For instance, \citet{drton2023algebraic} analyzed the degrees of freedom for sparse factor analysis, which is technically complex even with independent latent variables.

\subsection{Consistency and Exact Score-Based Search}\label{sec:hierarchical_structures_exact_search}
With the algorithm to compute the degrees of freedom, we now develop a score-based method to estimate latent hierarchical structures. We first establish the correctness of the score-based method. Under \cref{assumption:llh_graphical_criterion} and the generalized faithfulness assumption, we prove that the structure with the optimal score is Markov equivalent to the true hierarcahical structure, up to certain rank equivalent graph operators. The proof and definition of the operators together with illustrative examples can be found in \cref{app:thm_correctness_llh_exact}.
\begin{restatable}[Correctness]{theorem}{TheoremCorrectnessLLHExact}\label{thm:correctness_llh_exact}
Suppose that the true DAG $\mathcal{G}^*$ and the distribution $\Sigma_X$ satisfy the generalized faithfulness assumption, and that $\mathcal{G}^*$ satisfies \cref{assumption:llh_graphical_criterion}. Let $\hat{\mathcal{G}}$ be a global minimizer of the following optimization problem:
\begin{equation*}\label{eq:optimization_llh}
\begin{aligned}
\min_{\mathcal{G}\in\mathbb{G}^m}& \quad  \score_{\textrm{dim}}(\mathcal{G},\mathbf{D}) \\
\subjectto & \quad  \mathcal{G} \text{ satisfies  \cref{assumption:llh_graphical_criterion}},\\
& \quad  \mathcal{G}=\mathcal{O}_{\textrm{min}}(\mathcal{O}_{\textrm{skeleton}}(\mathcal{G})).
\end{aligned}
\end{equation*}
Then, 
$\mathcal{O}_{\textrm{atomic}}
(\hat{\mathcal{G}})$ and $\mathcal{O}_{\textrm{atomic}}(\mathcal{O}_{\textrm{min}}(\mathcal{O}_{\textrm{skeleton}}(\mathcal{G}^*)))$ are Markov equivalent in the large sample limit.
\end{restatable}
\begin{algorithm}[tb]
   \caption{Enumerating structures under \cref{assumption:llh_graphical_criterion}}
   \label{alg:enumerate_llh_structures}
\begin{algorithmic}
   \STATE {\bfseries Input:}  Measured variables $X_1,\dots,X_m$
   \STATE {\bfseries Output:} Set of DAGs $\set{A}$
   \STATE Initialize $\set{A}$ as an empty set;
   \FOR{$n=1$ {\bfseries to} $\bar{n}$} 
   \FOR{partition $\{\set{C}_i\}_{i=1}^l$ of $\{L_i\}_{i=1}^n$ as atomic covers}
   \FOR{DAG $\mathcal{G}_{\set{C}\rightarrow \set{C}}$ among $\{\set{C}_i\}_{i=1}^l$}
   \FOR{DAG $\mathcal{G}_{\set{C}\rightarrow X}$ from $\{\set{C}_i\}_{i=1}^l$ to $\{X_i\}_{i=1}^m$}
   \STATE Construct DAG $\mathcal{G}$ by combining $\mathcal{G}_{\set{C}\rightarrow \set{C}}$ and $\mathcal{G}_{\set{C}\rightarrow X}$;
   \IF{$\mathcal{G}$ satisfies \cref{assumption:llh_graphical_criterion} and is not Markov equivalent to all DAGs in $\set{A}$}
   \STATE $\set{A}\leftarrow\set{A}\cup\{\mathcal{G}\}$;
   \ENDIF
   \ENDFOR
   \ENDFOR
   \ENDFOR
   \ENDFOR
   \STATE {\bfseries return} set $\set{A}$
\end{algorithmic}
\end{algorithm}

\input{tables/f1_skeleton_table}

Based on the theorem above, we develop an exact search procedure for structures under \cref{assumption:llh_graphical_criterion}. We introduce a procedure in \cref{alg:enumerate_llh_structures} for enumeration of these structures,
where $\bar{n}$ is a hyperparameter indicating the maximal number of latent variables. Note that a possible upper bound for $\bar{n}$ is $3m$. Similar to the algorithm developed in \cref{alg:enumerate_silva_structures}, we enumerate only structures that are not Markov equivalent to one another, leveraging the score equivalence property. The whole procedure of  \cref{alg:enumerate_llh_structures} is roughly as follows. We maintain a set of DAGs, $\set{A}$. Given the number of observed variables, we first decide the possible number of latent variables, and then enumerate all possible combinations of atomic covers. For each combination of atomic covers, we enumerate all possible DAGs among atomic covers and all possible DAGs from atomic covers to observed variables. Finally, we combine both enumerated DAGs to get a possible graph $\mathcal{G}$; if $\mathcal{G}$ satisfies \cref{assumption:llh_graphical_criterion} and is not Markov equivalent to all structures in $\set{A}$, we add it to $\set{A}$. Once the search space is constructed, the algorithm identifies the structure with the optimal score, with the degrees of freedom for each structure computed using \cref{alg:count_llh_dimension}.

%% file: sections/3score.tex
\subsection{Formulation of Scoring Function}\label{sec:score}
We now provide the formulation of the scoring function for identifying linear latent variable causal models. Specifically, our score-based method involves searching for the structure with the smallest degrees of freedom that can generate the covariance matrix. Given structure $\mathcal{G}$ with samples $\mathbf{D}$ and empirical covariance matrix $S$, the scoring function is 
\[
\score_{\textrm{dim}}(\mathcal{G},\mathbf{D})\coloneqq  \begin{cases}
\dim(\mathcal{G}) & \text{ if } \mathcal{G} \text{ can generate } S, \\
\infty & \text{ otherwise.}
\end{cases}
\]
To determine whether structure $\mathcal{G}$ can generate $S$, one may minimize the squared errors between $S$ and the covariance matrix parameterized by $\mathcal{G}$, or compare the maximum likelihood w.r.t. $\mathcal{G}$ (e.g., see \cref{eq:score_likelihood}) to the likelihood of $S$. Moreover, we show in \cref{sec:bic_score} that the scoring function satisfies the property of score equivalence.

Similar scoring function has been discussed by~\citet{raskutti2014learning}. Roughly speaking, there may exist multiple structures that can generate the same distribution; the scoring function identifies one of them with the smallest degrees of freedom. In \cref{sec:score_identification,sec:1_factor_model,sec:hierarchical_structures}, we discuss how this scoring function identifies structures up to different types of model equivalence in the large sample limit. Specifically, we show in \cref{sec:score_identification} that it yields a structure that is algebraic equivalent to the ground truth.

Since the key idea is to identify the structure that generates the covariance matrix (in the large sample limit) with the smallest degrees of freedom, different types of scoring functions that can achieve so can also be used. In \cref{sec:bic_score}, we further discuss the use of the BIC score.

%% file: tables/f1_skeleton_table.tex
\begin{table*}[htbp]
\centering
\caption{F1 scores of skeletons across various structural assumptions and sample sizes. For each setting, the top two methods are in bold. For FOFC, the number within the brackets indicates the number of valid runs (for which an error did not occur).}
\label{tab:f1_skeletons}
\begin{small}
\begin{tabular}{cccccccc}
\toprule
Model type & Sample size  & SALAD & SALAD-CS & HUANG & FOFC & GIN \\\midrule
\multirow{5}{*}{\shortstack{1-Factor\\ \\ model}}
 & $100$   & {\bf 0.99\,$\pm$\,0.03} & {\bf 0.97\,$\pm$\,0.07} & 0.50\,$\pm$\,0.20 & 0.90\,$\pm$\,0.12 (8)  & 0.35\,$\pm$\,0.02 \\
 & $300$   & {\bf 0.99\,$\pm$\,0.01} & {\bf 0.99\,$\pm$\,0.02} & 0.85\,$\pm$\,0.15 & 0.98\,$\pm$\,0.03 (9)  & 0.35\,$\pm$\,0.02 \\
 & $1000$  & {\bf 0.99\,$\pm$\,0.01} & {\bf 0.99\,$\pm$\,0.01} & 0.93\,$\pm$\,0.03 & 0.98\,$\pm$\,0.03 (12) & 0.35\,$\pm$\,0.02 \\
 & $3000$  &    {\bf 1\,$\pm$\,0}    & {\bf 0.99\,$\pm$\,0.02} & 0.93\,$\pm$\,0.03 & 0.99\,$\pm$\,0.01 (12) & 0.35\,$\pm$\,0.02 \\
 & $10000$ &    {\bf 1\,$\pm$\,0}    & {\bf 0.99\,$\pm$\,0.02} & 0.93\,$\pm$\,0.03 & 0.99\,$\pm$\,0.01 (11) & 0.37\,$\pm$\,0.07 \\ \midrule
\multirow{5}{*}{\shortstack{Hierarchical\\ \\ structure}}
 & $100$   &  {\bf 0.92\,$\pm$\,0.06} & N/A & {\bf 0.57\,$\pm$\,0.13} & N/A (0) & 0.48\,$\pm$\,0.05 \\
 & $300$   &  {\bf 0.92\,$\pm$\,0.07} & N/A & {\bf 0.66\,$\pm$\,0.09} & N/A (0) & 0.48\,$\pm$\,0.05 \\
 & $1000$  &  {\bf 0.95\,$\pm$\,0.04} & N/A & {\bf 0.76\,$\pm$\,0.13} & N/A (0) & 0.48\,$\pm$\,0.05 \\
 & $3000$  &  {\bf 0.97\,$\pm$\,0.03} & N/A & {\bf 0.87\,$\pm$\,0.11} & N/A (0) & 0.48\,$\pm$\,0.05 \\
 & $10000$ &  {\bf 0.98\,$\pm$\,0.03} & N/A & {\bf 0.88\,$\pm$\,0.15} & N/A (0) & 0.47\,$\pm$\,0.05 \\
 \bottomrule
\end{tabular}
\end{small}
\end{table*}

%% file: sections/4experiments.tex
\section{Experiments}\label{sec:experiments}
We conduct experiments to validate our score-based methods, by comparing them to existing methods that support causally-related latent variables, such as FOFC~\citep{kummerfeld2016causal}, HUANG \citep{huang2022latent}, and GIN \citep{xie2020generalized}. We do not include the comparison with FCI because, even when working perfectly, it will output complete PAGs over the observed variables for most  ground truths considered here, which do not have any information of the orientation and are not informative. Moreover, we denote our exact search method by SALAD and continuous one by SALAD-CS, and adopt the BIC score here.

For the ground truths, we consider the 1-factor models and hierarchical structures provided in \cref{fig:silva_ground_truths,fig:llh_ground_truths} in \cref{app:experiment_details}. For each structure, the nonzero elements of matrices $B$ and $C$ are generated uniformly at random from the interval $[-2,-0.5]\cup[0.5,2.0]$. For GIN, the noise terms $E_X$ and $E_L$ follow $\operatorname{Uniform}[-\alpha,\alpha]$, where $\alpha$ is sampled uniformly from $[\sqrt{6},\sqrt{15}]$. For the other methods, the noise terms follow Gaussian distributions with variances sampled uniformly from interval $[2, 5]$. We consider sample size $T\in\{100,300,1000,3000,10000\}$. We evaluate the estimated structures using F1 scores calculated over the skeleton and structural Hamming distance (SHD) over the MEC. We run three random trials for each ground truth, and report the mean and standard devation for each metric. Further details about the metrics and baselines can be found in \cref{app:experiment_details}.\looseness=-1

The F1 scores of skeletons are reported in \cref{tab:f1_skeletons}, while the SHDs of MECs are given in \cref{tab:shd_pdags} in the supplementary material. One observes that our methods achieve much better F1 scores and SHDs as compared to the other baselines, especially for small sample sizes. For instance, for $100$ samples, our SALAD method achieves average F1 scores of $0.99$ and $0.92$ for 1-factor model and hierarchical structures, respectively, while the second best baseline achieves F1 scores of $0.90$ and $0.57$, respectively. Note that although FOFC achieves an F1 score of $0.90$ for 1-factor model in this case, four of the runs are not valid (i.e., an error occurred). A possible reason of the improvement is that the existing constraint-based baselines, as discussed in \cref{sec:intro}, may be prone to the issue of error propagation during the estimation procedure, while our score-based method is not susceptible to such an issue. Furthermore, the F1 scores of our method are close to one for both structural assumptions when the sample size is large, which suggest that BIC may be a valid scoring function in our setting and help verify the correctness established in \cref{thm:correctness_silva_exact,thm:correctness_llh_exact}. The runtime and computational efficiency are discussed in \cref{app:experiment_results}.

%% file: sections/5conclusion.tex
\section{Conclusion and Discussion}
In this work, we propose SALAD, a score-based causal discovery method capable of identifying causal relations among latent variables. Achieving score equivalence and consistency, along with degrees of freedom characterization and exact and continuous score-based methods, our work provides a unified view on multiple existing constraint-based methods with latent variables, and further validates the effectiveness of score-based methods. We hope that this work could spur future research on developing score-based methods for latent variable causal models.

Indeed, our exact methods require a relatively long runtime, similar to the exact score-based methods even without latent variables \citep{singh2005finding,yuan2013learning}. Future works include developing greedy approaches similar to GES to make the search procedure more efficient and scalable, and studying the theoretical justifications of using BIC score under the structural assumptions considered.

%% file: sections/acknowledgements.tex
\section*{Acknowledgements}

The authors would like to thank the anonymous reviewers for helpful comments and suggestions. The authors would also like to acknowledge the support from NSF Grant 2229881, the National Institutes of Health (NIH) under Contract R01HL159805, and grants from Apple Inc., KDDI Research Inc., Quris AI, and Florin Court Capital.

%% file: sections/impact_statement.tex
\section*{Impact Statement}
This paper presents work whose goal is to advance the field of Machine Learning. There are many potential societal consequences of our work, none which we feel must be specifically highlighted here.

%% file: sections/6appendices.tex
\section{Further Discussions}
We provide supplementary discussions below as complements to various sections in the main paper.
\subsection{Latent Variable Causal Models}\label{app:latent_variable_models}
In real-world scenarios, one may often encounter the latent variable causal model in \cref{eq:linear_sem}, where the measured variables do not influence each other and are effects of latent variables. Thus, there have been many works that aim to estimate this type of linear latent variable causal models~\citep{silva2003learning,silva2006learning,silva2005generalized,zhang2004hierarchical,choi2011learning,kummerfeld2016causal,cai2019triad,xie2020generalized,dai2022independence,huang2022latent,chen2022identification}.\looseness=-1

To provide some examples, in psychometrics, multiple questions are often used as indirect proxies for each latent personality dimension (e.g., openness, extraversion, self-esteem) \citep{lewis1992development,byrne2001structural,himi2019multitasking}, forming a latent variable causal model. When analyzing fMRI data, a large number of voxels are measured, which do not necessarily have clear semantic meanings. A hierarchical structure can then be used to model functionally meaningful brain regions at different levels \citep{huang2022latent}. In representation learning, recent works typically assumed that measured variables (e.g., image pixels) are effects of latent variables and that there are no direct causal influences among the measured variables~\citep{scholkopf2021causal,hyvarinen2023identifiability,zhang2024causal}.

\subsection{Generalized Faithfulness Assumption}\label{app:generalized_faithfulness}
We discuss the necessity of the generalized faithfulness assumption adopted in our results. One of the advantages of score-based causal discovery is that it typically relies on the sparsest Markov representation (SMR) assumption (or unique frugality assumption) \citep{forster2017frugal,raskutti2014learning}, which is strictly weaker than the faithfulness assumption~\citep{spirtes2001causation} in the setting without latent variables. In our setting with latent variables, it is possible to modify Theorems~\ref{thm:equivalence_equality_constraints}, \ref{thm:correctness_silva_exact}, and \ref{thm:correctness_llh_exact} to replace generalized faithfulness with a formulation similar to the SMR assumption. However, doing so may not be informative because (1) SMR may not be strictly weaker than the faithfulness assumption in our setting, and (2) simply assuming SMR in our setting may not provide insights into what structural assumptions the true structure should obey. Thus, we adopt the generalized faithfulness assumption and various structural assumptions to make the results more informative.

\subsection{Algebraic Equivalence}\label{app:algebraic_equivalence}
We provide a further discussion of algebraic equivalence~\citep{ommen2017algebraic} as a complement to~\cref{sec:score_identification}. First note that two algebraic equivalent structures are not necessarily Markov equivalent. The reason is that, without any restriction on the structures, one may construct different structures that entail the same equality constraints. For instance, consider the structure in \cref{fig:llh_ground_truths}(a), denoted as $\mathcal{G}_1$, and another structure $\mathcal{G}_2$ that is identical to $\mathcal{G}_1$, except that the edges $L_1\rightarrow X_1$, $L_2\rightarrow X_1$, and $L_1\rightarrow X_2$ are removed in $\mathcal{G}_2$. One can show $\mathcal{G}_1$ and $\mathcal{G}_2$ are algebraic equivalent, but clearly they are not Markov equivalent.

Nonetheless, algebraic equivalence may be a reasonable way for estimating linear latent variable causal models, because equality constraints (of the covariance matrices) are some of the major footprints in the data that one could leverage (without considering higher-order statistics) to identify the underlying structures. This can be done by relating these constraints to the structures via the generalized faithfulness assumption. 

\subsection{Structural Assumptions}\label{app:structural_assumptions}
We discuss the similarities and differences between the structural assumptions considered in our work. First, both \cref{assumption:silva_graphical_criterion,assumption:llh_graphical_criterion} require that the observed variables are leaf nodes, and that there are no direct causal influences among observed variables. The key differences between them are as follows. (1) \cref{assumption:silva_graphical_criterion} requires that each latent variable has at least three measured variables as its children, while \cref{assumption:llh_graphical_criterion} allows latent variables to form a hierarchical structure - some latent variables may only have latent variables as their children. (2) \cref{assumption:silva_graphical_criterion} requires each observed variable to be caused by a single latent variable, while \cref{assumption:llh_graphical_criterion} allows an observed variable to be caused by a group of latent variables. Since \cref{assumption:llh_graphical_criterion} does not require each latent variable to have measured variables as children, the structure among latent variables cannot be arbitrary, and thus there is a tradeoff between \cref{assumption:silva_graphical_criterion,assumption:llh_graphical_criterion}.

\section{Proofs}
\subsection{Proof of \cref{lemma:indeterminacy_Omega_L}}\label{app:proof_indeterminacy_omegaL}
\IndeterminacyOfOmegaL*
\vspace{-0.8em}
\begin{proof}
Let $\tilde{B}\coloneqq B\Omega_L^{\frac{1}{2}}$ and $\tilde{C}\coloneqq \Omega_L^{-\frac{1}{2}}C\Omega_L^{\frac{1}{2}}$. We have $\supp(B)=\supp(\tilde{B})$, $\supp(C)=\supp(\tilde{C})$, and
\begin{align*}
\Sigma_{X}&=B(I-C)^{-1}\Omega_L(I-C)^{-\top} B^{\top}+\Omega_X\\
&=B\Omega_L^{\frac{1}{2}} \Omega_L^{-\frac{1}{2}}(I-C)^{-1}\Omega_L^{\frac{1}{2}}\Omega_L^{\frac{1}{2}}(I-C)^{-\top} \Omega_L^{-\frac{1}{2}} \Omega_L^{\frac{1}{2}}B^{\top}+\Omega_X\\
&=(B\Omega_L^{\frac{1}{2}}) (I-\Omega_L^{-\frac{1}{2}}C\Omega_L^{\frac{1}{2}})^{-1}(I-\Omega_L^{-\frac{1}{2}}C\Omega_L^{\frac{1}{2}})^{-\top}(B\Omega_L^{\frac{1}{2}})^{\top}+\Omega_X\\
&=\tilde{B}(I-\tilde{C})^{-1}(I-\tilde{C})^{-\top} \tilde{B}^{\top}+\Omega_X\qedhere
\end{align*}
\end{proof}

\subsection{Proof of \cref{proposition:score_equivalence}}\label{app:proof_proposition_score_equivalence}
The following proof is partly inspired by that of the score equivalence property in the setting without latent variables \citep{koller2009probabilistic}.

\PropositionScoreEquivalence*
\begin{proof}
Because structures $\mathcal{G}_1$ and $\mathcal{G}_2$ are Markov equivalent, they can generate the same set of covariance matrices over variables $X$ and $L$. Thus, for any parameters $B,C,\Omega_X$ of $\mathcal{G}_1$ with $\Omega_L=I$, there exists parameters $B',C',\Omega_X',\Omega_L'$ of $\mathcal{G}_2$ that can generate the same covariance matrix over $X$ and $L$, which imply that $B',C',\Omega_X',\Omega_L'$ can generate the same covariance matrix over $X$. By \cref{lemma:indeterminacy_Omega_L}, there exists parameters of $\mathcal{G}_2$, denoted as $\tilde{B},\tilde{C},\tilde{\Omega}_X=\Omega_X'$  and $\tilde{\Omega}_L=I$, that can generate the covariance matrix. Note that the likelihood function depends only on the covariance matrix, which indicates
\[
\score_{\mathcal{L}}(\mathcal{G}_1,\mathbf{D})=\mathcal{L}(\hat{B},\hat{C},\hat{\Omega}_X;\mathbf{D})=\mathcal{L}(\tilde{B},\tilde{C},\tilde{\Omega}_X;\mathbf{D})\geq \score_{\mathcal{L}}(\mathcal{G}_2,\mathbf{D}),
\]
where $\hat{B},\hat{C},\hat{\Omega}_X$ are the solutions of the optimization problem in \cref{eq:score_likelihood} for $\mathcal{G}=\mathcal{G}_1$, and, as described above, $\tilde{B},\tilde{C},\tilde{\Omega}_X$ are the corresponding parameters of structure $\mathcal{G}_2$.

Similarly, the same reasoning implies $\score_{\mathcal{L}}(\mathcal{G}_2,\mathbf{D})\geq \score_{\mathcal{L}}(\mathcal{G}_1,\mathbf{D})$. Combining both cases, we have $\score_{\mathcal{L}}(\mathcal{G}_1,\mathbf{D})= \score_{\mathcal{L}}(\mathcal{G}_2,\mathbf{D})$. Furthermore, since $\mathcal{G}_1$ and $\mathcal{G}_2$ can generate the same set of covariance matrices over variables $X$ and $L$, they  can generate the same set of covariance matrices over variables $X$. This implies $\dim(\mathcal{G}_1)=\dim(\mathcal{G}_2)$. Therefore, we have $\score_{\textrm{BIC}}(\mathcal{G}_1,\mathbf{D})=\score_{\textrm{BIC}}(\mathcal{G}_2,\mathbf{D})$ and $\score_{\textrm{dim}}(\mathcal{G}_1,\mathbf{D})=\score_{\textrm{dim}}(\mathcal{G}_2,\mathbf{D})$.
\end{proof}

\subsection{Proof of \cref{thm:equivalence_equality_constraints}}\label{app:thm_equivalence_equality_constraints}
The overall proof strategy below is partly inspired by the proof of \citet[Theorem~3]{ghassami2020characterizing}.
\TheoremEquivalenceEqualityConstraints*
\vspace{-0.3em}
\begin{proof}
Since the search space contains the true DAG $\mathcal{G}^*$ that can generate $\Sigma_{X}$ in the large sample limit, the estimated DAG $\hat{\mathcal{G}}$ can also generate $\Sigma_{X}$, because otherwise its score will be infinity and will not be a solution of the optimization problem. Therefore, $\Sigma_X$ belongs to the distribution set of $\hat{\mathcal{G}}$, i.e., $\Sigma_X\in\mathcal{M}(\hat{\mathcal{G}})$, which implies that $\Sigma_X$ contains all the equality and inequality constraints of $\hat{\mathcal{G}}$. Under the generalized faithfulness assumption, we have 
\begin{equation}\label{eq:equalities_subset}
H(\hat{\mathcal{G}})\subseteq H(\mathcal{G}^*).
\end{equation}
Now suppose by contradiction that $H(\hat{\mathcal{G}})\subsetneq H(\mathcal{G}^*)$. This implies $\dim(\hat{\mathcal{G}})>\dim(\mathcal{G}^*)$, which is a contradiction because the objective function implies $\dim(\hat{\mathcal{G}})\leq \dim(\mathcal{G}^*)$. Thus, we obtain 
\begin{equation}\label{eq:equalities_not_subset}
H(\hat{\mathcal{G}})\not\subsetneq H(\mathcal{G}^*).
\end{equation}
By \cref{eq:equalities_subset,eq:equalities_not_subset}, we have $H(\hat{\mathcal{G}})= H(\mathcal{G}^*)$.
\end{proof}

\subsection{Proof of \cref{proposition:dimension_silva}}\label{app:proposition_dimension_silva}
We first state the following lemma adapted from \citet{leung2015identifiability} that relates the parameter identifiability from a given structure to the underlying degrees of freedom.
\begin{lemma}[{{\citet{leung2015identifiability}}}]\label{lemma:parameter_and_jacobian}
Suppose $f:\set{S}\rightarrow \mathbb{R}^d$ is a polynomial map defined on an open set $\set{S}\subseteq\mathbb{R}^p$. The following statements are equivalent:
\begin{enumerate}[label=(\roman*)]
\item $f$ is generically finite-to-one.
\item The Jacobian matrix of $f$ is generically of full column rank.
\end{enumerate}
\end{lemma}
We now provide the proof of the following proposition.
\PropositionDimensionSilva*
\begin{proof}
By~\cref{corollary:indeterminacy_Omega_L}, it suffices to consider the case where $\Omega_L=I$. Since the structure $\mathcal{G}$ satisfies \cref{assumption:silva_graphical_criterion}, by \citet{bollen1989general}, the corresponding parameters $B,C$ and $\Omega_X$ of $\mathcal{G}$ are identifiable from $\Sigma_X$ up to certain indeterminacy. Specifically, $B$ is identifiable up to column permutations and sign changes, $C$ is identifiable up to equal row and column permutations, and $\Omega_X$ is identifiable. Therefore, the map from $B,C$, and $\Omega_X$ to $\Sigma_X$ is generically finite-to-one.

The map from $B,C$, and $\Omega_X$ (with $\Omega_L=I$) to $\Sigma_X$ is a polynomial map. By \cref{lemma:parameter_and_jacobian}, the Jacobian matrix of such map is generically of full column rank. Note that  the degrees of freedom (or dimension) of a polynomial map are equal to the maximal rank of the corresponding Jacobian matrix \citep[Theorem~10]{geiger2001stratified}. Therefore, the degrees of freedom are equal to the number of parameters in $B,C$, and $\Omega_X$, i.e., $\operatorname{dim}(\mathcal{G})=\|B_\mathcal{G}\|_0+\|C_\mathcal{G}\|_0+m=|\mathcal{G}|+m$.
\end{proof}

\subsection{Proof of \cref{thm:correctness_silva_exact}}\label{app:thm_correctness_silva_exact}
\TheoremCorrectnessSilvaExact*
\begin{proof}
Since the search space contains the true DAG $\mathcal{G}^*$ that can generate $\Sigma_{X}$ in the large sample limit, the estimated DAG $\hat{\mathcal{G}}$ can also generate $\Sigma_{X}$, because otherwise its score will be infinity and will not be a solution of the optimization problem. Because $\mathcal{G}^*$ and $\Sigma_X$ satisfy the generalized faithfulness assumption, we have $H(\hat{\mathcal{G}})=H(\mathcal{G}^*)$ in the large sample limit by \cref{proposition:dimension_silva} and restricting the set of structures to those satisfying \cref{assumption:silva_graphical_criterion} in \cref{thm:equivalence_equality_constraints}. This indicates that $\hat{\mathcal{G}}$ and $\Sigma_X$ also satisfy the generalized faithfulness assumption.

Since $\hat{\mathcal{G}}$ and $\mathcal{G}^*$ satisfy \cref{assumption:silva_graphical_criterion} and both can faithfully generate the covariance matrix $\Sigma_X$, we have:
\begin{itemize}
\item By \citet[Corollary~1]{silva2003learning}, the measurement models of $\hat{\mathcal{G}}$ and $\mathcal{G}^*$ are identical (up to relabeling of latent variables). In other words, the columns of $B_{\hat{\mathcal{G}}}$ are a permutation of the columns of $B_{\mathcal{G}^*}$. 
\item With a correct measurement model, by leveraging the transitivity of Markov equivalence, it follows straightforwardly from \citet[Theorems~20]{silva2006learning} that the structural models (i.e., the subgraphs over all and only the latent variables) of $\hat{\mathcal{G}}$ and $\mathcal{G}^*$ are Markov equivalent (up to relabeling of latent variables).
\end{itemize}
Combining the reasoning for both measurement model and structural model, $\hat{\mathcal{G}}$ and $\mathcal{G}^*$ are Markov equivalent (up to relabeling of latent variables).
\end{proof}

\subsection{Proof of \cref{proposition:reduced_dimension}}\label{app:proposition_reduced_dimension}
For structure $\mathcal{G}$, we define the following distribution set with the constraint $\Omega_L=I$:
\[
\mathcal{M}(\mathcal{G};\Omega_L=I)\coloneqq\{B(I-C)^{-1}(I-C)^{-\top} B^{\top}+\Omega_X :\supp(B)\subseteq\supp(B_\mathcal{G}),\supp(C)\subseteq\supp(C_\mathcal{G}),\Omega_X\in\diag(\mathbb{R}_{> 0}^m)\}.
\]
Note that the dimension of the domain and the image space are upper bounds for the dimension of $\mathcal{M}(\mathcal{G};\Omega_L=I)$. By \cref{lemma:indeterminacy_Omega_L}, it is straightforward to obtain the following corollary.
\begin{corollary}\label{corollary:indeterminacy_Omega_L}
For any structure $\mathcal{G}$ satisfying \cref{eq:linear_sem}, we have
\[\mathcal{M}(\mathcal{G})=\mathcal{M}(\mathcal{G};\Omega_L=I) \text{\quad and\quad } \dim(\mathcal{G})\leq \min\left(|\mathcal{G}|+m,\frac{1}{2}m(m+1)\right).
\]
\end{corollary} 
Therefore, it suffices to analyze the degrees of freedom for $\mathcal{M}(\mathcal{G};\Omega_L=I)$ instead of $\mathcal{M}(\mathcal{G})$. 

We first provide the following lemma which shows that an appropriate orthogonal transformation of $B$ and $C$ can generate the same covariance matrix.
\begin{lemma}[Orthogonal transformation]\label{lemma:orthogonal_transformation}
Consider any set of parameters $B, C, \Omega_X$, and $\Sigma_X$ that satisfy
\[
\Sigma_X=B(I-C)^{-1}(I-C)^{-\top} B^{\top}+\Omega_X.
\]
For any orthogonal matrix $Q$, i.e., $QQ^{\top}=I$. the parameters $\tilde{B}=BQ$ and $\tilde{C}=Q^{\top}CQ$ also satisfy
\[
\Sigma_{X}=\tilde{B}(I-\tilde{C})^{-1}(I-\tilde{C})^{-\top} \tilde{B}^{\top}+\Omega_X.
\]
\end{lemma}
\begin{proof}
The proof follows from straightforward algebraic manipulations:
\begin{align*}
\Sigma_{X}&=B(I-C)^{-1}(I-C)^{-\top} B^{\top}+\Omega_X\\
&=BQ Q^{\top}(I-C)^{-1}QQ^{\top}(I-C)^{-\top} Q Q^{\top} B^{\top}+\Omega_X\\
&=(BQ) (I-Q^{\top}CQ)^{-1}(I-Q^{\top}CQ)^{-\top}(BQ)^{\top}+\Omega_X\\
&=\tilde{B}(I-\tilde{C})^{-1}(I-\tilde{C})^{-\top} \tilde{B}^{\top}+\Omega_X.\qedhere
\end{align*}
\end{proof}
The following result shows that, in specific cases, some of the edges can be removed from the structure while still leading to the same distribution set.
\begin{lemma}\label{lemma:removing_edges}
Suppose that DAG $\mathcal{G}$ follows the linear latent variable causal model in \cref{eq:linear_sem}. Suppose also that there exist $k\geq 2$ latent variables $\set{L}$ in $\mathcal{G}$ with the same set of parents and children, where the number of children (parents) is at least $k$. Then, there exists a structure, denoted by $\tilde{\mathcal{G}}$, such that: (1) $\tilde{\mathcal{G}}$ is identical to $\mathcal{G}$, except that $k(k-1)/2$ edges among those latent variables $\set{L}$ and their children (parents) are removed in $\tilde{\mathcal{G}}$, and (2)~$\mathcal{M}(\tilde{\mathcal{G}})=\mathcal{M}(\mathcal{G})$.
\end{lemma}
\begin{proof}
Consider any set of parameters $B, C, \Omega_X$, and $\Sigma_X$ that satisfy
\begin{equation}\label{eq:proof_parameterization}
\Sigma_X=B(I-C)^{-1}(I-C)^{-\top} B^{\top}+\Omega_X.
\end{equation}
We first consider the case where the number of children is at least $k$. Denote by $\set{S}$ the set of indices of the latent variables in $\set{L}$. Since the latent variables $\set{L}$ have the same set of children, the rows that correspond to the nonzero entries in each column of $B_{:, \set{S}}$ and $C_{:, \set{S}}$ are the same, which we denote by $\set{R}_1$ and $\set{R}_2$ for $B_{:, \set{S}}$ and $C_{:, \set{S}}$, respectively. Let $D=(B_{\set{R}_1, \set{S}},C_{\set{R}_2, \set{S}})$ be a matrix by concatenating the rows of $B_{\set{R}_1, \set{S}}$ and $C_{\set{R}_2, \set{S}}$; that is, $D$ is a matrix of dimension $(|\set{R}_1|+|\set{R}_2|)\times k$, where $|\set{R}_1|+|\set{R}_2|\geq k$. Applying orthogonal transformation as in the QR-decomposition, $D$ can be written as $D=\tilde{D}Q$, where $\tilde{D}$ is a lower-triangular matrix and $Q$ is an orthogonal matrix. We rewrite the equation as $\tilde{D}=DQ^{-1}$ where $Q^{-1}$ is also an orthogonal matrix.

Consider the reversed mapping of the indices $\set{A}_1\coloneqq\{1,2,\dots,|\set{R}_1|\}$ and $\set{A}_2\coloneqq\{|\set{R}_1|+1,|\set{R}_1|+2,\dots,|\set{R}_1|+|\set{R}_2|\}$. We now construct an $n\times n$ orthogonal matrix $U$ as follows: (1) $U_{\set{S},\set{S}}=Q^{-1}$, (2) the other non-diagonal entries are zero, and (3) the other diagonal entries are one. Let $\tilde{B}=BU$ and $\tilde{C}=U^{\top}CU$. Clearly, the entries in $\tilde{B}$ are the same as $B$, except that $B_{\set{R}_1, \set{S}}$ is replaced with $B_{\set{R}_1, \set{S}}U_{\set{S},\set{S}}=B_{\set{R}_1, \set{S}}Q^{-1}=(DQ^{-1})_{\set{A}_1,:}=\tilde{D}_{\set{A}_1,:}$. Similarly, the entries in $\tilde{C}$ are the same as $C$, except that (1) $C_{\set{R}_2, \set{S}}$ is replaced with $C_{\set{R}_2, \set{S}}U_{\set{S},\set{S}}=C_{\set{R}_2, \set{S}}Q^{-1}=(DQ^{-1})_{\set{A}_2,:}=\tilde{D}_{\set{A}_2,:}$, and (2) $C_{\set{S}, :}$ is replaced with $U_{\set{S},\set{S}}^{\top}C_{\set{S}, :}$. Since the latent variables $\set{L}$ have the same set of parents, we have $\supp(U_{\set{S},\set{S}}^{\top}C_{\set{S}, :})\subseteq \supp(C_{\set{S}, :})$. This implies that $\tilde{B},\tilde{C}$, and $\Omega_X$ are parameterization of, e.g., structure $\tilde{\mathcal{G}}$, where $\tilde{\mathcal{G}}$ has the same edges as $\mathcal{G}$, except that $k(k-1)/2$ of the edges from $\mathcal{G}$ are removed in $\tilde{\mathcal{G}}$ (which correspond to the ``non-lower-triangular'' entries from $D$ that become zero in $\tilde{D}$ after QR-decomposition). Clearly, $\tilde{\mathcal{G}}$ is identical to $\mathcal{G}$, except that $k(k-1)/2$ edges among those latent variables $\set{L}$ and their children are removed in $\tilde{\mathcal{G}}$. Furthermore, by \cref{lemma:orthogonal_transformation}, the parameters $\tilde{B}$, $\tilde{C}$, and $\Omega_X$  can generate the same covariance matrix $\Sigma_X$.

Since we are able to construct the same structure $\tilde{\mathcal{G}}$ using the above procedure for every parameterization $B$, $C$, and $\Omega_X$ of $\mathcal{G}$ in \cref{eq:proof_parameterization}, we have $\mathcal{M}(\tilde{\mathcal{G}};\Omega_L=I)=\mathcal{M}(\mathcal{G};\Omega_L=I)$, which, by \cref{corollary:indeterminacy_Omega_L}, implies $\mathcal{M}(\tilde{\mathcal{G}})=\mathcal{M}(\mathcal{G})$. The same reasoning also applies when the number of parents is at least $k$, i.e., such a structure $\tilde{\mathcal{G}}$ can also be constructed.
\end{proof}

We now provide the proof of the following proposition.
\PropositionReducedDimension*
\begin{proof}
By \cref{lemma:removing_edges}, there exists a structure, denoted by $\tilde{\mathcal{G}}$, such that: (1) $\tilde{\mathcal{G}}$ is identical to $\mathcal{G}$, except that $k(k-1)/2$ edges are removed in $\tilde{\mathcal{G}}$, and (2)~$\mathcal{M}(\tilde{\mathcal{G}})=\mathcal{M}(\mathcal{G})$. By \cref{corollary:indeterminacy_Omega_L}, this implies
\begin{align*}
\dim(\mathcal{G})=\dim(\tilde{\mathcal{G}})\leq |\tilde{\mathcal{G}}|+m=|\mathcal{G}|+m-\frac{1}{2}k(k-1).
\end{align*}
\end{proof}

\subsection{Proof of \cref{proposition:dimension_llh}}\label{app:proposition_dimension_llh}
\PropositionDimensionLLH*
\begin{proof}
Let $\set{L}_1,\set{L}_2,\dots,\set{L}_p$ be the pairwise disjoint sets of latent variables in structure $\mathcal{G}$ such that (1) each $\set{L}_i$ has at least two latent variables, (2) the variables of each $\set{L}_i$ have the same set of parents and children in $\mathcal{G}$, (3) no proper superset of each $\set{L}_i$ has the same set of parents and children $\mathcal{G}$. Since the structure $\mathcal{G}$ is a DAG, we assume without loss of generality that $\set{L}_1,\set{L}_2,\dots,\set{L}_p$ are sorted based on the reversed causal ordering in $\mathcal{G}$. That is, variables $\set{L}_{i_1}$ cannot be the ancestors of variables $\set{L}_{i_2}$ in structure $\mathcal{G}$ for $i_1<i_2$.

Although \cref{alg:count_llh_dimension} does not impose any specific order on the sets of latent variables, we suppose that the algorithm computes the dimension based on $\set{L}_1,\set{L}_2,\dots,\set{L}_p$ (sorted according to reversed causal ordering), which does not affect the computed degrees of freedom. Denote by $d_j$ the output of the algorithm in the $j$-th iteration where $j\in[p]$. The final output is then $d_p$. It suffices to show $\dim(\mathcal{G})\leq d_j$ in each iteration.

We provide a proof by induction. Specifically, we show that, for the $j$-th iteration, there exists a structure $\tilde{\mathcal{G}}_j$ such that:
\begin{enumerate}
\item $\tilde{\mathcal{G}}_j$ is identical to $\mathcal{G}$, except that $|\mathcal{G}|+m-d_j$ edges among the variables $\bigcup_{i=1}^j\set{L}_i$ and their children are removed in $\tilde{\mathcal{G}}_j$.
\item $\mathcal{M}(\tilde{\mathcal{G}}_j)=\mathcal{M}(\mathcal{G})$.
\end{enumerate}
By \cref{corollary:indeterminacy_Omega_L}, this implies the desired outcome
\[
\dim(\mathcal{G})=\dim(\tilde{\mathcal{G}}_j)\leq |\tilde{\mathcal{G}}_j|+m=|\mathcal{G}|+m-(|\mathcal{G}|+m-d_j)=d_j.
\]

For induction, we first consider the base case $j=1$. By assumption, the variables $\set{L}_1$ have the same set of parents and children in $\mathcal{G}$, where, under \cref{assumption:llh_graphical_criterion}, the number of children is at least $|\set{L}_1|$. By \cref{lemma:removing_edges}, there exists a structure, denoted by $\tilde{\mathcal{G}}_1$, such that: (1) $\tilde{\mathcal{G}}_1$ is identical to $\mathcal{G}$, except that $|\set{L}_1|(|\set{L}_1|-1)/2=|\mathcal{G}|+m-d_1$ edges among the variables $\set{L}_1$ and their children are removed in $\tilde{\mathcal{G}}_1$, and (2)~$\mathcal{M}(\tilde{\mathcal{G}}_1)=\mathcal{M}(\mathcal{G})$. Therefore, the base case is done.

Suppose that the statements hold for $j=t$, i.e., there exists a structure $\tilde{\mathcal{G}}_t$ such that: (1) $\tilde{\mathcal{G}}_t$ is identical to $\mathcal{G}$, except that $|\mathcal{G}|+m-d_t$ edges among the variables $\bigcup_{i=1}^t\set{L}_i$ and their children are removed in $\tilde{\mathcal{G}}_t$, and (2)~$\mathcal{M}(\tilde{\mathcal{G}}_t)=\mathcal{M}(\mathcal{G})$.

Now consider $j=t+1$. Note that only the edges among $\bigcup_{i=1}^t\set{L}_i$ and their children are removed in $\tilde{\mathcal{G}}_t$ (as compared to $\mathcal{G}$); by assumption, the variables $\bigcup_{i=1}^t\set{L}_i$ are not ancestors of the variables $\set{L}_{t+1}$. Therefore, the incoming and outgoing edges of variables $\set{L}_{t+1}$ in $\tilde{\mathcal{G}}_t$ are the same as those in $\mathcal{G}$. This implies that $\set{L}_{t+1}$ have the same set of parents and children in $\tilde{\mathcal{G}}_t$, because they have the same set of parents and children in $\mathcal{G}$. Furthermore, under \cref{assumption:llh_graphical_criterion}, the number of children is at least $|\set{L}_{t+1}|$. By \cref{lemma:removing_edges}, there exists a structure, denoted by $\tilde{\mathcal{G}}_{t+1}$, such that: (1) $\tilde{\mathcal{G}}_{t+1}$ is identical to $\tilde{\mathcal{G}}_t$, except that $|\set{L}_{t+1}|(|\set{L}_{t+1}|-1)/2$ edges among the variables $\set{L}_{t+1}$ and their children are removed in $\tilde{\mathcal{G}}_{t+1}$, and (2)~$\mathcal{M}(\tilde{\mathcal{G}}_{t+1})=\mathcal{M}(\tilde{\mathcal{G}}_t)$. 

By the induction hypothesis, we have $\mathcal{M}(\tilde{\mathcal{G}}_{t+1})=\mathcal{M}(\tilde{\mathcal{G}}_t)=\mathcal{M}(\mathcal{G})$. Also, it is clear that $\tilde{\mathcal{G}}_{t+1}$ is identical to $\mathcal{G}$, except that
\[
|\mathcal{G}|+m-d_t+\frac{1}{2}|\set{L}_{t+1}|(|\set{L}_{t+1}|-1)=|\mathcal{G}|+m-d_{t+1}
\]
edges among the variables $\bigcup_{i=1}^{t+1}\set{L}_i$ and their children are removed in $\tilde{\mathcal{G}}_{t+1}$. Therefore, the induction step is done.
\end{proof}

\begin{figure}[b]
  \vspace{-0mm}
  \centering
\centering
\subfloat[Graph $\mathcal{G}$.]{
\includegraphics[width=0.46\textwidth]{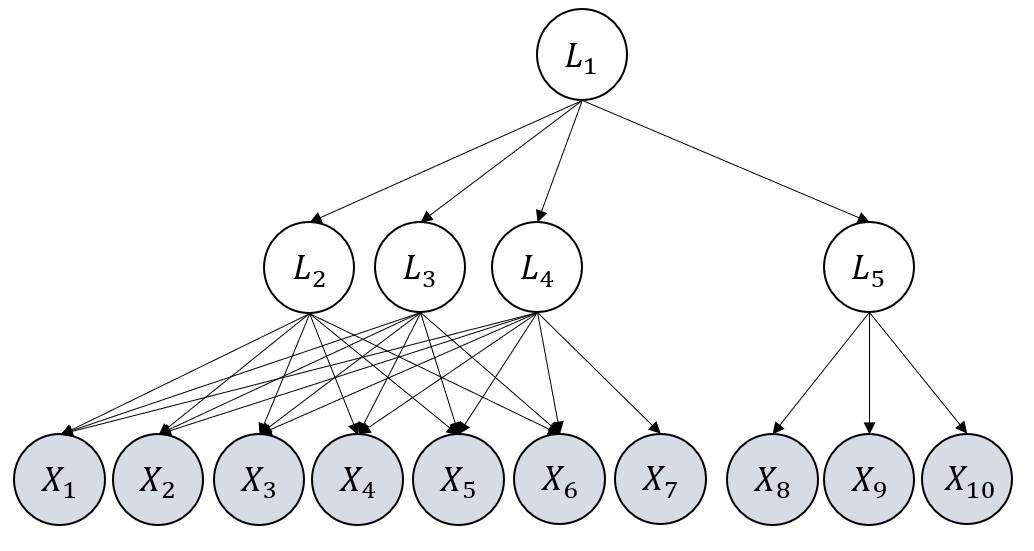}
}
\centering
\subfloat[$\mathcal{O}_{\text{skeleton}}(\mathcal{G})$.]{
    \includegraphics[width=0.46\textwidth]{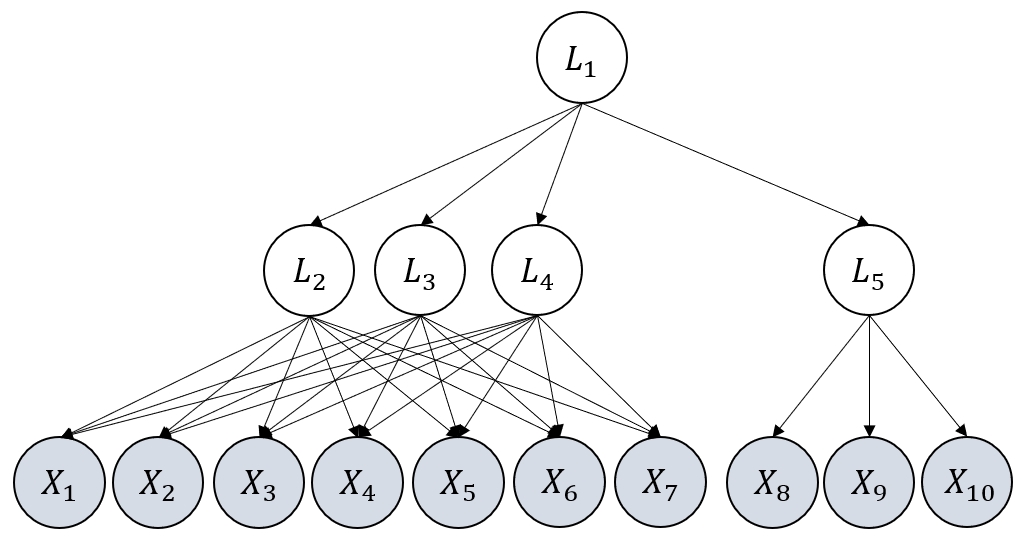}
}\hspace{1.3em}
\subfloat[$\mathcal{O}_{\text{min}}(\mathcal{O}_{\text{skeleton}}(\mathcal{G}))$.]{
    \includegraphics[width=0.46\textwidth]{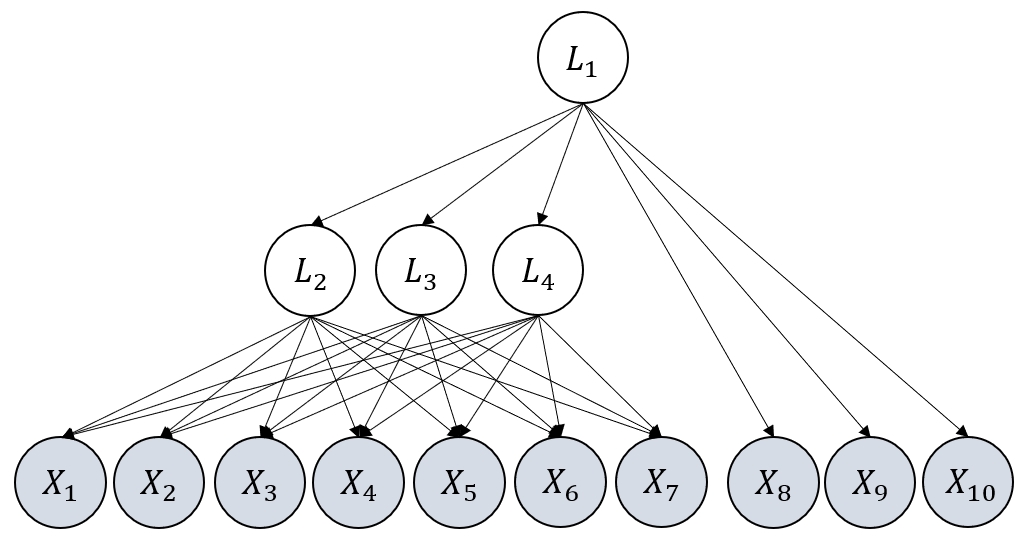}
}\hspace{1.3em}
\subfloat[$\mathcal{O}_{\text{atomic}}(\mathcal{O}_{\text{min}}(\mathcal{O}_{\text{skeleton}}(\mathcal{G})))$.]{
    \includegraphics[width=0.46\textwidth]{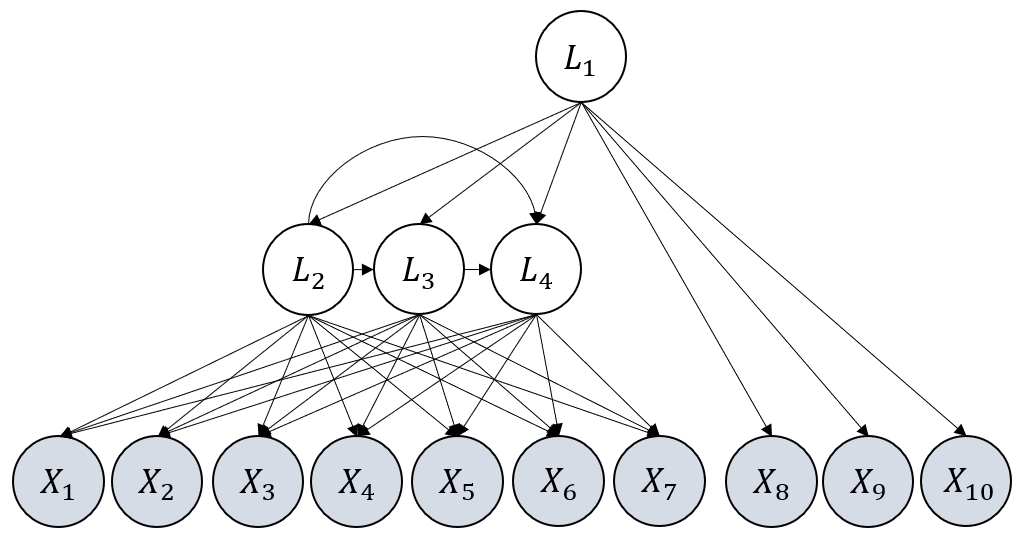}
}\hspace{1.3em}
  \caption{Example to illustrate graph operators $\mathcal{O}_{\text{atomic}}$, $\mathcal{O}_{\text{min}}$, and $\mathcal{O}_{\text{skeleton}}$.}
  \label{fig:example for operator}
\end{figure}

\subsection{Definition of Graph Operators and Proof of \cref{thm:correctness_llh_exact}}\label{app:thm_correctness_llh_exact}

We provide the definition of structure operations
$\mathcal{O}_{\text{atomic}}$, $\mathcal{O}_{\text{min}}$, and $\mathcal{O}_{\text{skeleton}}$ below, with an example in \cref{fig:example for operator}.

  \begin{definition}[Minimal-graph operator \citep{huang2022latent,dong2023versatile}]
    \label{def:minimal operator}
  For every two atomic covers $\set{L}$ and $\set{P}$ in structure $\graph$, we merge $\set{L}$ to $\set{P}$ if the following conditions hold: (i) $\set{L}$ is the pure children of $\set{P}$, (ii) all elements of $\set{L}$ and $\set{P}$ are latent and  $|\set{L}| = |\set{P}|$, and (iii) the pure children of $\set{L}$ form a single atomic cover, or the siblings of $\set{L}$ form a single atomic cover. We denote such an operator as minimal-graph operator $\mathcal{O}_{\text{min}}(\graph)$.
  \end{definition}
  
    \begin{definition}[Skeleton operator \citep{huang2022latent,dong2023versatile}]
    \label{def:skeleton operator}
    Given an atomic cover $\set{L}$ in structure $\graph$.  Consider $\setset{S}$ as the set of atomic covers such that for all $\set{S}\in \setset{S}$, we have $\set{S}\subseteq\set{L}$. Let $\set{C}= \purechildren(\set{L}) \backslash \cup_{\set{S}\in\setset{S}}\purechildren(\set{S})$. We add edges from elements in $\set{L}$ to elements in $\set{C}$, and 
    we denote such an operator as skeleton operator $\mathcal{O}_{\text{skeleton}}(\graph)$.
  \end{definition}

 \begin{definition}[Intra atomic operator]
    \label{def:intra atomic operator}
    For every atomic cover $\set{L}$ in structure $\graph$, if $|\set{L}|\geq 2$, then we add edges between elements in $\set{L}$ such that  $\set{L}$ form a fully connected DAG. We denote such an operator as intra atomic operator $\mathcal{O}_{\text{atomic}}(\graph)$.
  \end{definition}

\begin{example}[Example for graph operations]
Let the graph in \cref{fig:example for operator}(a) be $\graph$.
By the skeleton operator,
we add edges from $L_2$ and $L_3$ to $X_7$,
and we arrive at $\mathcal{O}_{\text{skeleton}}(\graph)$,
which is shown in \cref{fig:example for operator}(b).
By the minimal graph operator,
we delete $L_5$ and directly link $L_1$ to $X_8,X_9,X_{10}$,
and arrive at $\mathcal{O}_{\text{min}}(\mathcal{O}_{\text{skeleton}}(\graph))$,
which is shown in \cref{fig:example for operator}(c).
Finally,
by the intra atomic operator,
we add edges among $L_2,L_3,L_4$ such that they are fully connected,
and arrive at $\mathcal{O}_{\text{atomic}}(\mathcal{O}_{\text{min}}(\mathcal{O}_{\text{skeleton}}(\graph)))$,
which is shown in \cref{fig:example for operator}(d).
\end{example}

\begin{remark}[Necessity of graph operators]
These three graph operators do not change the rank constraints (among measured variables), and thus by using rank constraints for causal discovery as in \citet{huang2022latent,dong2023versatile}, we can at most identify the structure up to these graph operators.
\end{remark}

We now provide the proof of the following result.
\TheoremCorrectnessLLHExact*
\begin{proof}
Since the search space contains the DAG $\mathcal{O}_{\textrm{min}}(\mathcal{O}_{\textrm{skeleton}}(\mathcal{G}^*))$ that can generate $\Sigma_{X}$ in the large sample limit, the estimated DAG $\hat{\mathcal{G}}$ can also generate $\Sigma_{X}$, because otherwise its score will be infinity and will not be a solution of the optimization problem.  Because $\mathcal{G}^*$ and $\Sigma_X$ satisfy the generalized faithfulness assumption, we have $H(\hat{\mathcal{G}})=H(\mathcal{G}^*)$ in the large sample limit by restricting the set of structures to those satisfying \cref{assumption:llh_graphical_criterion} and $\mathcal{G}=\mathcal{O}_{\textrm{min}}(\mathcal{O}_{\textrm{skeleton}}(\mathcal{G}))$ in \cref{thm:equivalence_equality_constraints}. This indicates that $\hat{\mathcal{G}}$ and $\Sigma_X$ also satisfy the generalized faithfulness assumption.

Let $\mathcal{G}'$ be the structure estimated by Algorithm~1 in \citet{huang2022latent} based on the covariance matrix $\Sigma_X$. Since $\hat{\mathcal{G}}$ satisfies \cref{assumption:llh_graphical_criterion} and can faithfully generate $\Sigma_X$, \citet[Theorem~10]{huang2022latent} and \citet[Theorem~13]{dong2023versatile} imply that $\mathcal{O}_{\textrm{atomic}}(\mathcal{G}')$ and $\mathcal{O}_{\textrm{atomic}}(\mathcal{O}_{\textrm{min}}(\mathcal{O}_{\textrm{skeleton}}(\hat{\mathcal{G}})))$ are Markov equivalent. With similar reasoning, we can show that $\mathcal{O}_{\textrm{atomic}}(\mathcal{G}')$ and $\mathcal{O}_{\textrm{atomic}}(\mathcal{O}_{\textrm{min}}(\mathcal{O}_{\textrm{skeleton}}(\mathcal{G}^*)))$ are Markov equivalent. Therefore, by the transitivity of Markov equivalence, $\mathcal{O}_{\textrm{atomic}}(\mathcal{O}_{\textrm{min}}(\mathcal{O}_{\textrm{skeleton}}(\hat{\mathcal{G}})))$ and $\mathcal{O}_{\textrm{atomic}}(\mathcal{O}_{\textrm{min}}(\mathcal{O}_{\textrm{skeleton}}(\mathcal{G}^*)))$ are Markov equivalent.

Recall that $\hat{\mathcal{G}}=\mathcal{O}_{\textrm{min}}(\mathcal{O}_{\textrm{skeleton}}(\hat{\mathcal{G}}))$. This implies that $\mathcal{O}_{\textrm{atomic}}(\hat{\mathcal{G}})$ and $\mathcal{O}_{\textrm{atomic}}(\mathcal{O}_{\textrm{min}}(\mathcal{O}_{\textrm{skeleton}}(\mathcal{G}^*)))$ are Markov equivalent.
\end{proof}

\section{Supplementary Experiment Details}\label{app:experiment_details}
\textbf{Implementation details.} \ \
To improve the efficiency of \cref{alg:count_llh_dimension}, we iterate over the latent atomic covers to identify latent variables with the same set of parents and children. For our exact search method, we use L-BFGS \citep{byrd1995limited} implemented through \texttt{SciPy} \citep{virtanen2020scipy} and \texttt{PyTorch} \citep{paszke2019pytorch} packages (with the default hyperparameters) to solve the optimization problem in \cref{eq:score_likelihood} when computing BIC. The experiments for the exact search method are conducted on $16$ CPUs in parallel. For the continuous search method, i.e., SALAD-CS, we use augmented Lagrangian method \citep{bertsekas1982constrained,bertsekas1999nonlinear,nocedal2006numerical} to solve the continuous constrained optimization problem, in which each subproblem is solved using the Adam optimizer \citep{kingma2014adam} with $3000$ iterations. Furthermore, \cref{eq:continuous_optimization_silva} involves a nonconvex optimization problem; similar to continuous optimization methods for causal discovery \citep{ng2024structure}, this procedure may yield suboptimal local solutions. Thus, we run the SALAD-CS method from $10$ random initializations, and select the final solution with the best score.

For HUANG, GIN, and RCD, we use the publicly available implementations with default hyperparameters. For FOFC, we use the implementation through the \texttt{py-causal} package \citep{scheines1998tetrad} with Wishart test and significance  level of $0.001$. Note that we also experimented with significance level of $0.01$, $0.05$ and $0.1$, for which many of the runs are invalid (because an error occurred).

\textbf{Metrics.} \ \
Since the goal is to recover the structure up to Markov equivalence, we compute the SHDs over the MECs. Note that the labeling of the latent variables is not important; therefore, we calculate the SHDs of the estimated MECs over all possible permutations of the latent variables, and select the smallest SHD. Similarly, we also compute the F1 scores of the estimated skeletons over all permutations of latent variables, and select the highest F1 score.

For FOFC, an error occurred in some of the experimental runs. Therefore, we additionally report the number of valid runs (for which an error did not occur).

\section{Runtime and Computational Efficiency}\label{app:experiment_results}
In this section, we report the runtime for different methods considered. For the 1-factor models, our SALAD method has a runtime of $8.77\pm 0.73$ and $44.88\pm 8.02$ minutes for $10$ and $11$ measured variables, respectively, while for hierarchical structures, it takes $16.11\pm 2.01$ minutes. For the SALAD-CS method, each optimization run takes $14.17\pm 0.69$ and $14.73\pm 1.91$ minutes for $10$ and $11$ measured variables, respectively. For the baselines, GIN, HUANG, and FOFC generally finish within one minute. For the 1-factor models, RCD requires $17.23\pm 34.84$ and $13.06\pm 24.33$ minutes for $10$ and $11$ measured variables, respectively, while for hierarchical structures, it has a runtime of $7.76\pm 15.33$ minutes.

Our methods have a comparable runtime as RCD, but achieve better performance. Although the runtime of our methods exceeds that of GIN, HUANG, and FOFC, the improvement in the causal discovery performance is significant. As discussed in \cref{app:experiment_details}, our experiments are conducted on CPUs. It is worth noting that the runtime may be further decreased by (i)~conducting experiments with GPU acceleration (specifically when using gradient-based optimization to solve \cref{eq:score_likelihood,eq:continuous_optimization_silva}), or (ii) performing more score computations of different structures (specifically for exact search) concurrently on different CPUs.

Indeed, the relatively long runtime of our methods may be unsurprising because, even without latent variables, exact score-based methods \citep{singh2005finding,yuan2013learning} are known to require a long runtime. The search procedure developed in our work serves as a proof of concept, tailored for scenarios involving a relatively small number of variables. Nonetheless, the empirical performance validates the effectiveness of score-based methods for estimating latent variable causal models. Future works include developing greedy approaches similar to GES to make the search procedure more efficient and scalable.

\input{tables/shd_mec_table}

\clearpage
\begin{figure}[!t]
\centering
\subfloat[Example 1.]{
    \includegraphics[width=0.46\textwidth]{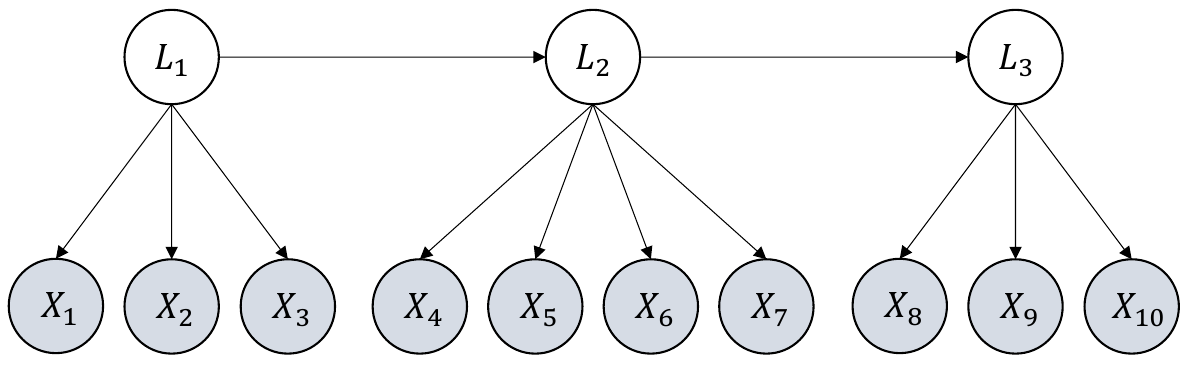}
}\hspace{1.3em}
\subfloat[Example 2.]{
    \includegraphics[width=0.46\textwidth]{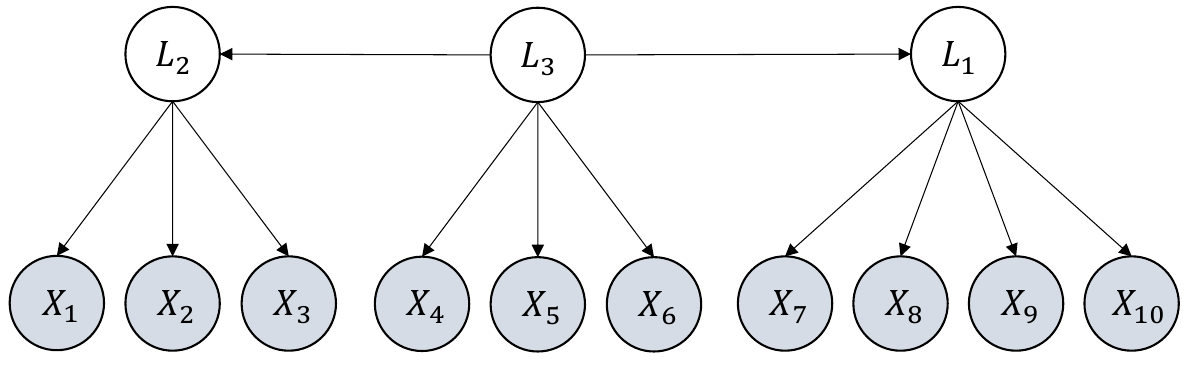}
}\\
\subfloat[Example 3.]{
    \includegraphics[width=0.49\textwidth]{figures/silva_example_3.pdf}
}\hspace{1.3em}
\subfloat[Example 4.]{
    \includegraphics[width=0.42\textwidth]{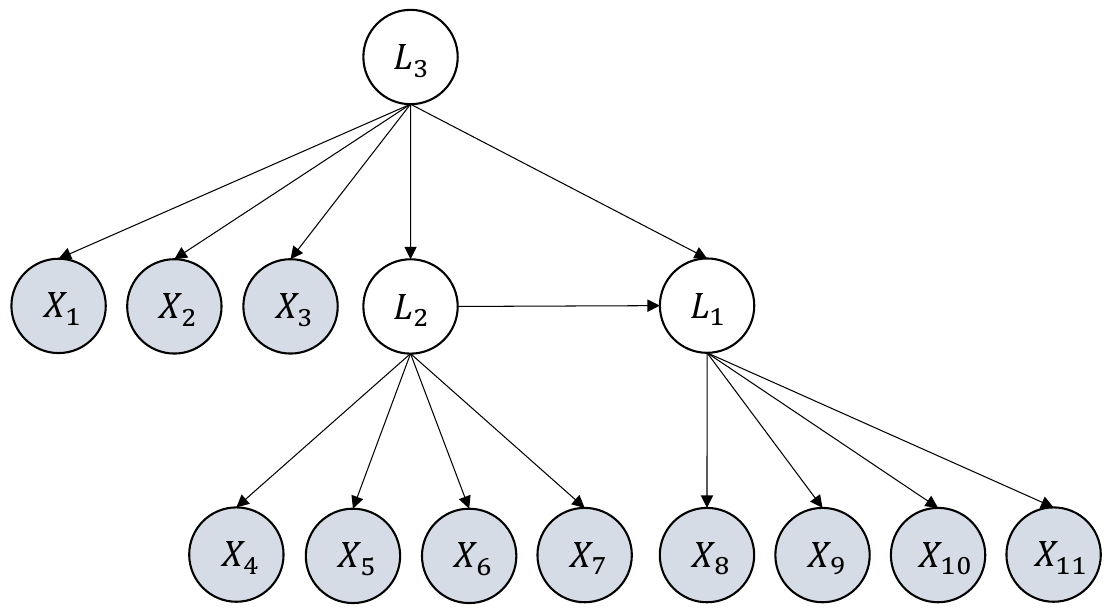}
}
\caption{Ground truths for 1-factor latent variable models.}
\label{fig:silva_ground_truths}
\end{figure}

\begin{figure}[!t]
\centering
\subfloat[Example 1.]{
    \includegraphics[width=0.4\textwidth]{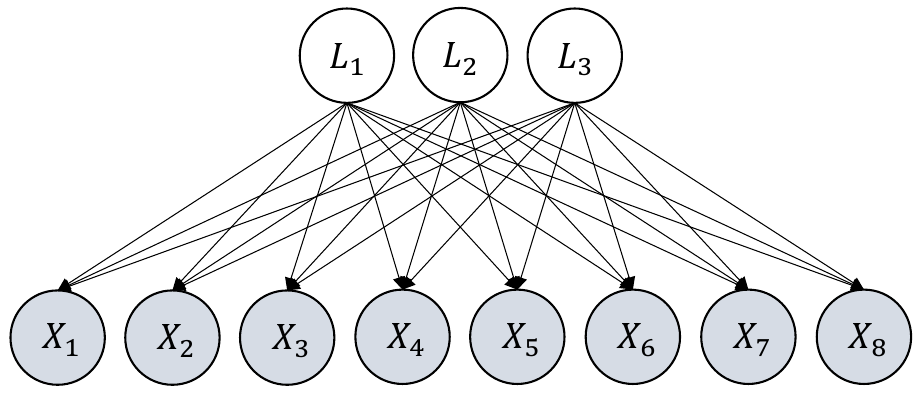}
}\hspace{1.3em}
\subfloat[Example 2.]{
    \includegraphics[width=0.3\textwidth]{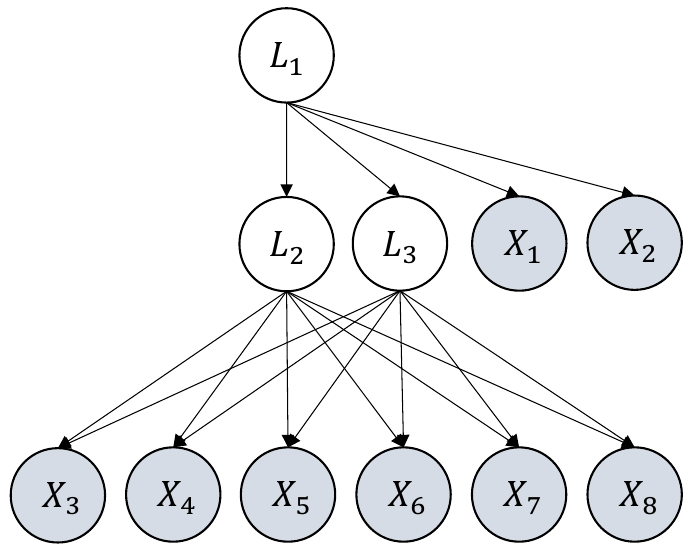}
}\\
\subfloat[Example 3.]{
    \includegraphics[width=0.4\textwidth]{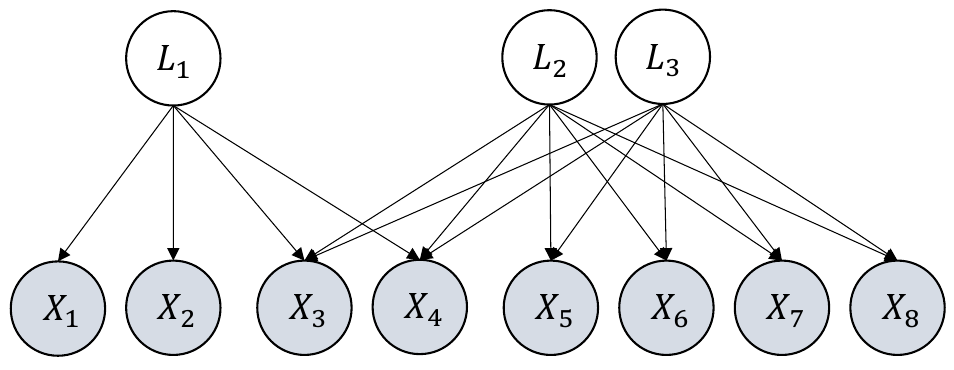}
}\hspace{1.3em}
\subfloat[Example 4.]{
    \includegraphics[width=0.36\textwidth]{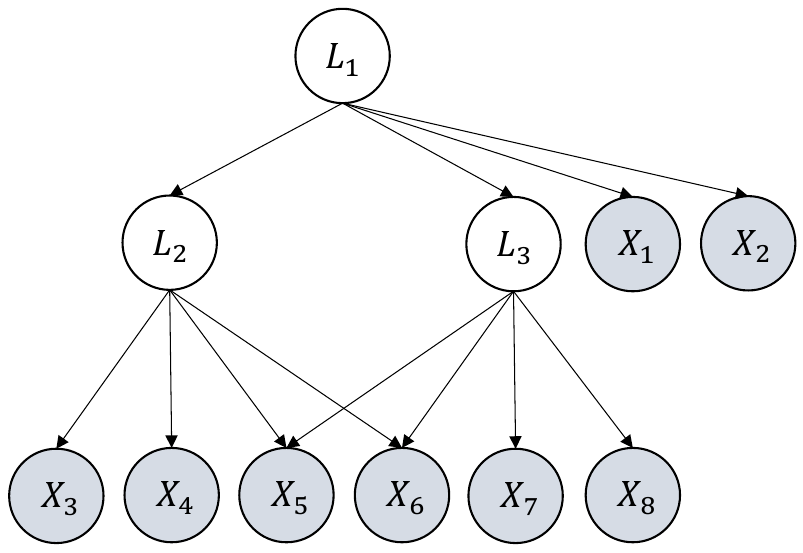}
}
\caption{Ground truths for latent hierarchical structures.}
\label{fig:llh_ground_truths}
\end{figure}

%% file: tables/shd_mec_table.tex
\begin{table*}[!h]
\centering
\caption{SHDs of MECs across various structural assumptions and sample sizes. For each setting, the top two methods are in bold. For FOFC, the number within the brackets indicates the number of valid runs (for which an error did not occur).}
\label{tab:shd_pdags}
\begin{tabular}{cccccccc}
\toprule
Model type & Sample size  & SALAD & SALAD-CS & HUANG & FOFC & GIN \\\midrule
\multirow{5}{*}{\shortstack{1-Factor\\ \\ model}}
 & $100$   & {\bf 0.33\,$\pm$\,0.65} &  {\bf 0.75\,$\pm$\,1.76} & 11.50\,$\pm$\,4.58 & 3.13\,$\pm$\,2.64 (8)  &  15.25\,$\pm$\,1.36 \\
 & $300$   & {\bf 0.08\,$\pm$\,0.29} &  {\bf 0.17\,$\pm$\,0.39} & 3.50\,$\pm$\,3.45  & 1.67\,$\pm$\,1.11 (9)  &  15.25\,$\pm$\,1.36 \\
 & $1000$  & {\bf 0.33\,$\pm$\,0.89} &  {\bf 0.08\,$\pm$\,0.29} & 1.67\,$\pm$\,0.78  & 1.67\,$\pm$\,1.07 (12) &  15.25\,$\pm$\,1.36 \\
 & $3000$  &    {\bf 0\,$\pm$\,0}    &  {\bf 0.17\,$\pm$\,0.39} & 1.67\,$\pm$\,0.78  & 1.25\,$\pm$\,1.14 (12) &  15.17\,$\pm$\,1.34 \\
 & $10000$ &    {\bf 0\,$\pm$\,0}    &  {\bf 0.17\,$\pm$\,0.39} & 1.67\,$\pm$\,0.78  & 1.18\,$\pm$\,1.17 (11) &  14.75\,$\pm$\,2.01 \\ \midrule
\multirow{5}{*}{\shortstack{Hierarchical\\ \\ structure}}
 & $100$   &  {\bf 4.50\,$\pm$\,3.32} & N/A & {\bf 14.25\,$\pm$\,2.73} & N/A (0) & 18.50\,$\pm$\,3.73 & \\
 & $300$   &  {\bf 3.67\,$\pm$\,3.06} & N/A & {\bf 11.42\,$\pm$\,3.20} & N/A (0) & 18.50\,$\pm$\,3.73 & \\
 & $1000$  &  {\bf 2.75\,$\pm$\,2.18} & N/A &  {\bf 8.00\,$\pm$\,3.64} & N/A (0) & 18.50\,$\pm$\,3.73 & \\
 & $3000$  &  {\bf 1.92\,$\pm$\,1.98} & N/A &  {\bf 4.92\,$\pm$\,4.08} & N/A (0) & 18.50\,$\pm$\,3.73 & \\
 & $10000$ &  {\bf 1.42\,$\pm$\,1.56} & N/A &  {\bf 4.17\,$\pm$\,4.91} & N/A (0) & 18.42\,$\pm$\,3.75 & \\
 \bottomrule
\end{tabular}
\end{table*}